\documentclass[twocolumn,10pt]{asme2ej}
\usepackage[font=footnotesize]{subcaption}
\usepackage{epsfig} %% for loading postscript figures
\usepackage{hyperref}
\usepackage{bm}
\usepackage{xcolor}

\usepackage{amssymb}
\usepackage{amsmath}
\newcommand{\comment}[1]{}
\newtheorem{proposition}{Proposition}
\newtheorem{remark}{Remark} 	
\newtheorem{definition}{Definition} 
 
\DeclareMathOperator*{\argmax}{arg\,max}

\title{Modeling and Prediction of Rigid Body Motion with Planar Non-Convex Contact}

\author{Jiayin Xie and Nilanjan Chakraborty
    \affiliation{
	Department of Mechanical Engineering\\
	Stony Brook University\\
	Stony Brook, New York 11790\\
    Email: \{jiayin.xie,~nilanjan.chakraborty\}@stonybrook.edu
    }	
}

\begin{document}

\maketitle    

%%%%%%%%%%%%%%%%%%%%%%%%%%%%%%%%%%%%%%%%%%%%%%%%%%%%%%%%%%%%%%%%%%%%%%
\begin{abstract}
 {\it We present a principled method for motion prediction via dynamic simulation for rigid bodies in intermittent contact with each other where the contact region is a planar non-convex contact patch.
 Such methods are useful in planning and control for robotic manipulation. The planar non-convex contact patch can either be a topologically connected set or disconnected set.  Most work in rigid body dynamic simulation assume that the contact between objects is a point contact, which may not be valid in many applications. In this paper,  by using the convex hull of the contact patch, we build on our recent work on simulating rigid bodies with convex contact patches for simulating motion of objects with planar non-convex contact patches. We formulate a discrete-time mixed complementarity problem where we solve the contact detection and integration of the equations of motion simultaneously. We solve for the equivalent contact point (ECP) and contact impulse of each contact patch simultaneously along with the state, i.e., configuration and velocity of the objects. We prove that although we are representing a patch contact by an equivalent point, our model for enforcing non-penetration constraints ensure that there is no artificial penetration between the contacting rigid bodies.  We provide empirical evidence to show that our method can seamlessly capture transition among different contact modes like patch contact, multiple or single point contact.   }
\end{abstract}

\section{Introduction}
\label{se:intro}

\begin{figure}[t]
\includegraphics[width=0.45\textwidth]{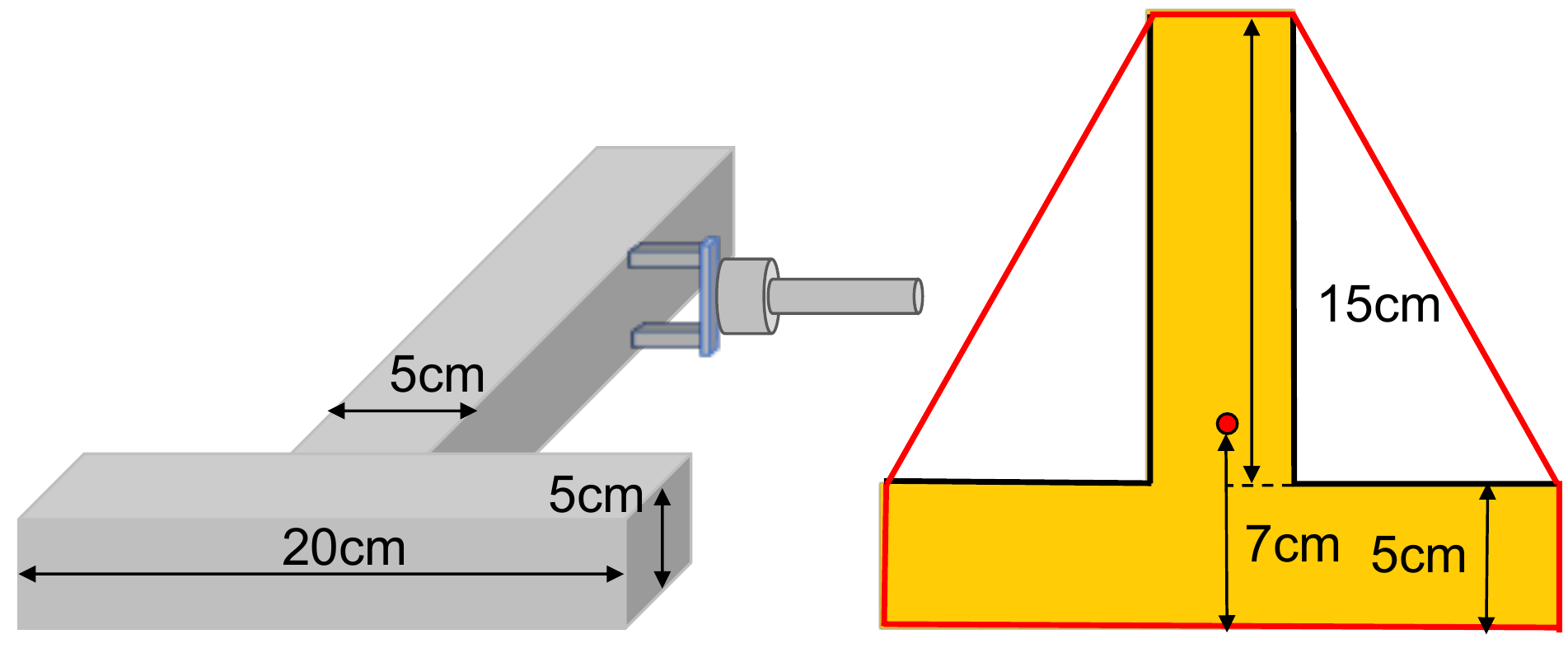}%
\caption{(Left) A T-shaped bar on planar surface is manipulated by a gripper while being supported on the plane. (Right) The planar contact between the bar and support surface is a non-convex T-shaped patch. The red line shows the convex hull for the contact patch.}
\label{figure_Motivation} 
\end{figure}
Rigid body motion prediction via dynamic simulation is a key enabling technology in solving many exemplar robotic manipulation tasks like manipulation with multi-fingered hands, manipulation with vibratory plates, automated parts feeder design, and design of microrobots~\cite{ReznikC98,SongTVP04,VoseUL09,BerardNAT10,XieB+2019}.
Many of these manipulation tasks involve point and surface contacts between the rigid body that is being manipulated and a flat plane on which the body rests~\cite{ReznikC98,VoseUL09,DafleR+14}. Furthermore, the occurrence of multiple intermittent contacts makes the prediction of the motion more complicated. There are applications in which the contact between two objects may be over a patch that can be modeled as a non-convex set. For example, Figure~\ref{figure_Motivation} shows a robot manipulator manipulating a T-shaped bar where the contact between the ground and the bar is a planar non-convex set. Such situations may arise when a robot manipulator with a parallel jaw gripper is trying to reconfigure a heavy bar with support from the table, so that it does not have to support the full weight. State-of-the-art dynamic simulation algorithms that can be used to predict motions during planning, usually assume point contact between two objects (except~\cite{XieC16, XieC18a}), which is clearly violated in Figure~\ref{figure_Motivation}. There are no well-principled approaches to predict the effect of applying a force/torque on the bar. In this paper, we seek to develop principled algorithms for predicting motion of rigid bodies in intermittent contact where the contacts can be modeled as a planar non-convex set.

%Dynamic simulation for the motion of rigid body plays an important role in robotic manipulation, which includes in-hand manipulation like grasping~\cite{ChavanR15, MaD11} and non-prehensile manipulation as pushing~\cite{lynch1996stable}. There exists intermittent contacts between the manipulator (e.g., gripper or pusher) and the rigid body. In many cases, the contact between them could be planar contact, where the patch for the contact is a subset of a plane. Furthermore, when the geometry of the planar contact is complex, it would makes the prediction of the motion more complicated. There exists one example for the manipulation task with complex geometry of planar contact. As Figure~\ref{figure_Motivation} illustrates, a robot hand manipulates a computer mouse on the planar support, where the contact between the mouse feet and the support is a union of rectangle and a crescent surface.  

Figure~\ref{figure:different_contact_case} shows the key types of contact between objects. Most existing mathematical models for motion of objects with intermittent contact like Differential Algebraic Equation (DAE) models~\cite{Haug1986} and Differential Complementarity Problem (DCP) models~\cite{Cottle2009,Trinkle1997,PfeifferG08} assume the contact between the two objects is a single point contact (top left in Figure~\ref{figure:different_contact_case}). However, for convex contact patch (middle row in Figure~\ref{figure:different_contact_case}), the point contact assumption is not valid. In such case, multiple contact points are usually chosen in an ad hoc manner, which can lead to inaccuracies in simulation (Please see~\cite{XieC16} for example scenarios). Recently, we developed an approach~\cite{XieC16} to simulate contacting rigid bodies with convex contact patches (line and surface contact). In~\cite{XieC18a}, we develop an approach for simulating contacting bodies where the contact patch is non-convex but can be modeled as a union of convex sets (bottom row, right column in Figure~\ref{figure:different_contact_case}).  In this paper, we focus on simulating bodies with planar non-convex contact patch, where the non-convex contact patch may not be a union of convex sets. The contact can be multiple point contacts or a general planar non-convex patch contact (top row, right column and bottom row in Figure~\ref{figure:different_contact_case}). Such situations arise when a robot is manipulating objects supported by a horizontal plane.

\begin{figure}[h]
\centering
\includegraphics[width=0.4\textwidth]{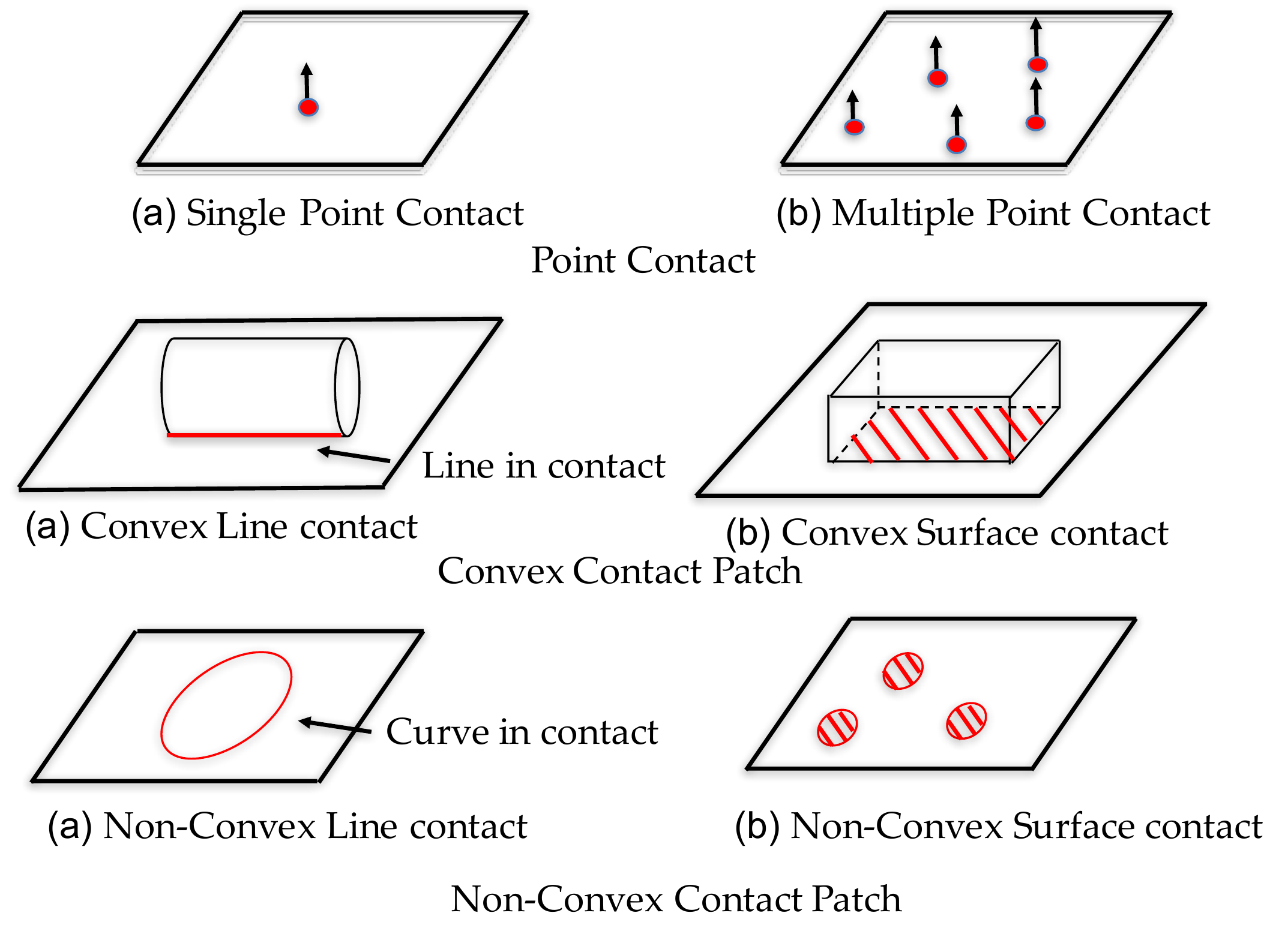}
\caption{Different types of contact between an object and a flat surface. Our focus is on simulating rigid bodies with type of contact shown in last row and first row, pane (b).}
\label{figure:different_contact_case} 
\end{figure}

For a single convex contact patch, we know that there exists a unique point on the contact surface where the integral of total moment due to normal force acting on this point is zero. This point is used to model line or surface contact as a point contact and thus it is called the {\em equivalent contact point} (ECP)~\cite{XieC16}. Using the concept of ECP, in~\cite{XieC16}, we present a principled method for simulating intermittent contact with convex contact patches (line and surface contact). This method solves for the ECP as well as the contact impulses by {\em incorporating  the collision detection within the dynamic simulation time step}. This method is called the {\em geometrically implicit time-stepping method} because the geometric information of contact points and contact normal are solved as a part of the numerical integration procedure. In~\cite{XieC18a}, for non-convex contact patches that can be modeled as a union of convex sets, we use an ECP to model the effect of each convex contact patch and solve for the ECP and its associated contact wrenches on each contact patch separately. However, the limitation of this method was that the force/moment distribution and the ECP was non-unique, although the state of the object was unique. Furthermore, if there are more than three convex sets forming the non-convex patch, the force/moment in some of the contact patches may become zero.

In this paper, we extend the method in~\cite{XieC16}, by using the convex hull of the contact patch for modeling the contact constraints in the equations of motion. Although, we have intermittent contact and the contact patch may change (even topologically, we can go from a connected non-convex patch to multiple point contact), we do not need to form the convex hull of the contact patch during the simulation depending on the contact mode. Instead, we use the convex hull of the non-convex object that is being manipulated. And since we solve the collision detection problem simultaneously with the equations of motion (i.e., our method is geometrically implicit), we can ensure that the {\em convex hull of the contact patch} will always be automatically obtained through our contact detection constraints. Note that distinct from~\cite{XieC16}, the ECP may not be a point within the physical contact region (but it will be a point within the convex hull of the contact region). We prove that even though we are modeling a non-convex contact patch with an equivalent contact point that may not lie within the patch, the contact constraints are always satisfied at the end of the time-step and there is no artificial penetration between the objects. We show simulation results validating our approach with our previous models~\cite{XieC18a,XieC18b}. We also present simulation results showing that the object can seamlessly transition among different contact modes like non-convex patch contact, multiple point contact, line contact, and single point contact. A preliminary version of this work was presented in~\cite{XieC2019}. We extend on the paper in~\cite{XieC2019}, by including complete proofs of Proposition~\ref{pro:1} and Proposition~\ref{pro:2}. We have also extended the simulation results section by including more simulation results.

\section{Related Work}
\label{se:re}
In this section, we present the related work in rigid body dynamic simulation with a focus on methods for dealing with intermittent contact. There is also a substantial body of work on development of discretization schemes for integrating and simulating rigid body motion that we do not discuss here (please see the literature on variational integrators~\cite{MarsdenJM01,JohnsonRM09,KobilarovCD09} and references therein). 
%(Related work about variational integrator~\cite{MarsdenJM01,JohnsonRM09})
We model the continuous time dynamics of rigid bodies that are in intermittent contact with each other as a Differential Complementarity Problem (DCP).  
Let ${\bf u}\in \mathbb{R}^{n_1}$,  ${\bf v}\in \mathbb{R}^{n_2}$ and let ${\bf g}$ :$ \mathbb{R}^{n_1}\times \mathbb{R}^{n_2} \rightarrow \mathbb{R}^{n_1} $, ${\bf f}$ : $ \mathbb{R}^{n_1}\times \mathbb{R}^{n_2} \rightarrow \mathbb{R}^{n_2}$ be two vector functions. 
%and the notation $0 \le \bm{x} \perp \bm{y} \ge 0$ imply that $\bm{x}$ is orthogonal to $\bm{y}$ and each component of the vectors is non-negative. 
\begin{definition}
Let ${\bf x}, {\bf y} \in \mathbb{R}^n$ be two vectors, with $x_i, y_i$ as the $i$th component of ${\bf x}$ and ${\bf y}$ respectively. The vectors ${\bf x}$ and ${\bf y}$ are said to satisfy a complementarity constraint if
$$ x_iy_i = 0, ~x_i \geq 0, ~y_i \geq 0, ~\forall i. $$
Equivalently, each component of the vectors ${\bf x}$ and ${\bf y}$ is non-negative and ${\bf x}$ is orthogonal to ${\bf y}$.
A shorthand notation for the complementarity constraints is $0 \le {\bf x} \perp {\bf y} \ge 0$.
%which imply that $\bm{x}$ is orthogonal to $\bm{y}$ and each component of the vectors is non-negative.
\end{definition}
\begin{definition}
The differential (or dynamic) complementarity problem~\cite{Facchinei2007} is to find ${\bf u}$ and ${\bf v}$ satisfying
$$\dot{\bf u} = {\bf g}({\bf u},{\bf v}), \ \ \ 0\le {\bf v} \perp {\bf f}({\bf u},{\bf v}) \ge 0 $$
\end{definition}
\begin{definition}
The mixed complementarity problem is to find ${\bf u}$ and ${\bf v}$ satisfying
$${\bf g}({\bf u},{\bf v})=0, \ \ \ 0\le {\bf v} \perp {\bf f}({\bf u},{\bf v}) \ge 0.$$
If the functions ${\bf f}$ and ${\bf g}$ are linear, the problem is called a mixed linear complementarity problem (MLCP), otherwise, the problem is called a mixed nonlinear complementarity problem (MNCP). Our continuous time dynamics model is a DCP whereas our discrete-time dynamics model is a MNCP. 
\end{definition}

The DCP model formulates the intermittent contact between bodies in motion as a complementarity constraint~\cite{Lotstedt82, AnitescuCP96, Pang1996, StewartT96, Liu2005, PfeifferG08, DrumwrightS12, Todorov14, Studer2009, CapobiancoE2018,BrulsAC2018}. 
%A substantial amount of effort in modeling and dynamic simulation with complementarity constraints~\cite{AnitescuCP96, Pang1996, StewartT96, Liu2005, Drumwright2012, Todorov2014}. 
DCP models are solved numerically with time-stepping schemes. The time-stepping problem is: {\em given the state of the system and applied forces, compute an approximation of the system one time step into the future. } Solving this problem repeatedly will give an approximate solution to the equations of motion. 
When a fixed-time stepping scheme is used to solve a DCP, it is usually implicit in the formulation that the collision between two objects is perfectly inelastic or plastic. Since we will be using a fixed time-stepping scheme, we also assume that the collision between two objects is perfectly inelastic. However, note that it is possible to remove the assumption of plastic collision within a complementarity framework (please see~\cite{AnitescuP02,NilanjanChakraborty2007}). In general, collision and impact laws for rigid body motion has been widely studied. A few references in this direction are~\cite{Brogliato2000, Jia2013,Tavakoli+2012,ChatterjeeRuina1998}. 
%We use a backward Euler discretization to form the discrete time-stepping problem. For relevant literature on the effect of time-step selection on accuracy of integration of equations of motion, please see the literature on variational integrators

There are different assumptions for forming the discrete equations of motion, which makes the discrete-time system Mixed Linear Complementarity problem (MLCP)~\cite{AnitescuP97, AnitescuP02} or mixed non-linear complementarity problem (MNCP)~\cite{Tzitzouris01,NilanjanChakraborty2007}. The MLCP problem linearizes the friction cone constraints and the distance function between two bodies (which is a nonlinear function of the configuration), sacrificing accuracy for speed. Depending on whether the distance function is approximated, the time-stepping schemes can also be divided into geometrically explicit schemes~\cite{AnitescuCP96, StewartT96} and geometrically implicit schemes~\cite{Tzitzouris01}. 

In geometrically explicit schemes, at the current state, a collision detection routine is called to determine separation or penetration distances between the  bodies, but this information is not incorporated as a function of the unknown future state at the end of the current time step. A goal of a typical time-stepping scheme is to guarantee consistency of the dynamic equations and all model constraints at the end of each time step. However, since the geometric information is obtained and approximated only at the start of the current time-step, then the solution will be in error. Apart from being geometrically explicit, most of the existing complementarity-based dynamic simulation methods and software also assume point contact between objects~\cite{CouBullet,SmithODE,TodorovET2012,TasoraS+2015,LeeG+2018,BerardT+2007}. A patch contact is usually approximated by ad hoc choice of $3$ contact points on the contact patch. In~\cite{XieC16}, we compared our non-point contact model with two popular point-based models, namely, Open Dynamic Engine (ODE)~\cite{SmithODE} and Bullet~\cite{CouBullet} in a pure translation task with a square contact patch where the analytic closed-form solution is known. We showed that our results matched the theoretical results, and was more accurate compared to ODE and Bullet. Thus, in~\cite{XieC16,NilanjanChakraborty2007}, we used a geometrically implicit time stepping scheme for solving convex contact patches problem, which is also the method used in this paper. The resulting discrete time problem is a MNCP.
\section{Dynamic Model for Rigid Body Systems}
\label{se:dynamic}
We will now formulate the equations of motion of rigid objects moving with intermittent contact as a differential complementarity problem (DCP) for continuous time and as a nonlinear complementarity problem (NCP) for discrete time. 
%We make the following modeling assumptions: (1) the objects are solids and rigid bodies (2) the collision between the objects
%In complementarity methods, the dynamic simulation of intermittent unilateral contact between two rigid objects can be modeled by a geometrically implicit optimization-based time-stepping scheme. 
The dynamic model is made up of the following parts: (a) Newton-Euler equations (b) kinematic map relating the generalized velocities to the linear and angular velocities (c) friction law and (d) non-penetration constraints.  
The parts (a), (b) form a system of ordinary differential equations~\cite{rao2005dynamics} and they are standard for any complementarity-based formulation. Part (c) can be written as a system of complementarity constraints, which is based on Coulomb friction law using the maximum work dissipation principle. Part (d) incorporates the geometry of contact set as  system of complementarity constraint~\cite{NilanjanChakraborty2007,XieC16,XieC18a}. 
%Note that the contact between objects is a planar contact patch, which can be either convex or non-convex.
\subsection{Equivalent Contact Point (ECP)}
\label{sec:ECP}
The contact between two objects can be point contact or non-point (i.e., patch) contact. Furthermore, the patch contact can be planar patch contact or non-planar patch contact. In this paper, we assume that the contact is planar patch contact (which includes point contact as a special case). Planar patch contact can be either convex patch contact or non-convex patch contact and the non-convex patch contact can be union of disconnected contact regions. Figure~\ref{figure:Equivalent_contact_point} gives a schematic sketch of a convex contact patch and a non-convex contact patch.

Irrespective of the geometry of the contact patch, the normal contact force that prevents penetration of the two objects is distributed over the contact patch. From basic physics, we know that there will be  a point in the convex hull of the contact patch such that the moment of the normal force about the point is $0$. We call this point the {\em equivalent contact point (ECP)} of the contact patch. The ECP along with the equivalent contact wrench (due to distributed normal force as well as distributed friction force over the contact patch) that acts at this point so that the two objects do not penetrate is unique. Note that the ECP does not necessarily lie within the contact patch, although it will lie in the convex hull of the contact patch (see Figure~\ref{figure:Equivalent_contact_point}, where the convex hull is the red curve and the ECP is the red point). 

In this section, we will formulate our equations of motion in terms of the ECP and the equivalent contact wrench acting at the ECP. We will also present algebraic and complementarity constraints that allows computation of the ECP, contact wrench as well as the state of the objects in a discrete-time framework. Note that in our method, there is {\em no assumption made on the nature of the pressure distribution between the two surfaces}. The pressure distribution was used to define the notion of ECP, but it is not required for the computation of ECP and equivalent contact wrench. We will show that the discrete-time equations of motion gives a contact wrench acting at the ECP such that the non-penetration between the two objects is always guaranteed.

%So we can potentially use the implicit time-stepping method to solve for the ECP as the closest point on the contact surface, its associated wrench and configurations of the object simultaneously. In the subsequent sections, we show that it is indeed the case and prove that the guarantee of non-penetration in~\cite{NilanjanChakraborty2007} that is valid for point contact between two objects can be extended to line and surface contact.

\begin{figure}[ht]
\centering
\includegraphics[width=0.5\textwidth]{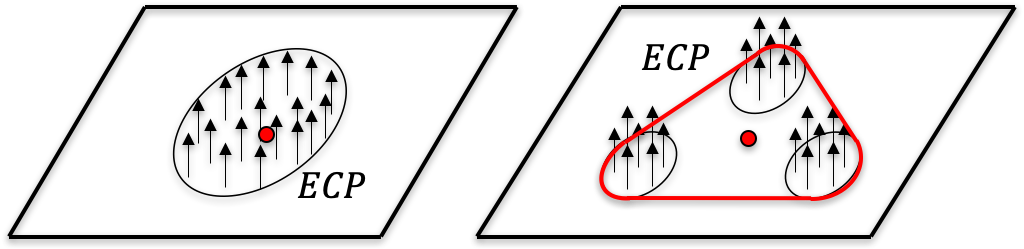}
\caption{Schematic sketch of the normal force distribution on contact patches that prevents penetration and associated ECPs. (Left) The ECP in single convex contact patch. (Right) The ECP for non-convex contact patch which does not lie within the patch but lies in the convex hull of the contact patch.}
\label{figure:Equivalent_contact_point} 
\end{figure} 

\subsection{Newton-Euler equations of motion}

To describe the dynamic model mathematically, we will introduce some notation first. Let ${\bf q}$ be the position of the center of mass of the object and the orientation of the object (${\bf q}$ can be $6 \times 1$ or $7\times 1$ vector depending on the representation of the orientation). We will use unit quaternion to represent the orientation unless otherwise stated. The generalized velocity $\bm{\nu}$ is the concatenated vector of linear (${\bf v}$) and spatial angular ($^s\bm{\omega}$) velocities. \comment{In general, the planar contact patch is composed of $n_c$ number of individual topologically disconnected regions or patches.} The effect of the contact patch is modeled as point contact of equivalent contact points (ECPs) ${\bf a}_1$ or ${\bf a}_2$ on two objects. Let
$\lambda_n$ be the magnitude of normal contact force,
$\lambda_t$ and $\lambda_o$ be the orthogonal components of the friction force on the tangential plane, and
$\lambda_r$ be the frictional moment about the contact normal.

\begin{equation} 
\begin{aligned}
\label{eq1}
{\bf M}({\bf q})
{\dot{\bm{\nu}}} = 
{\bf W}_{n}\lambda_{n}+
{\bf W}_{t} \lambda_{t}+
{\bf W}_{o} \lambda_{o}
+{\bf W}_{r}\lambda_{r}+
\bm{\lambda}_{app}+\bm{\lambda}_{vp}
\end{aligned}
\end{equation}
where ${\bf M}(\bf{q})$ is the generalized inertia matrix.  $\bm{\lambda}_{app}$ is the  vector of external forces (including gravity) and moments, $\bm{\lambda}_{vp}$ is the  vector of Coriolis and centripetal forces. The unit wrenches ${\bf W}_{n}$, ${\bf W}_{t}$, ${\bf W}_{o}$ and ${\bf W}_{r}$ are dependent on configuration ${\bf q}$ and ECP (${\bf a}_{1}$ or ${\bf a}_{2}$), and map the normal contact forces, frictional forces and moments to the inertia frame:
\begin{equation}
\begin{aligned}
\label{equation:wrenches}
&{\bf W}_{n} =  \left [ \begin{matrix} 
{\bf n}\\
{\bf r}\times {\bf n}
\end{matrix}\right]
\quad
{\bf W}_{t} =  \left [ \begin{matrix} 
{\bf t}\\
{\bf r}\times {\bf t}
\end{matrix}\right]
\\
&{\bf W}_{o} =  \left [ \begin{matrix} 
{\bf o}\\
{\bf r}\times {\bf o}
\end{matrix}\right]
\quad 
{\bf W}_{r} =  \left [ \begin{matrix} 
\ {\bf 0}\\
\ {\bf n} 
\end{matrix}\right]
\end{aligned}
\end{equation}
where $({\bf n},{\bf t},{\bf o})$ are unit vectors of contact frame and ${\bf r}$ is the vector from center of mass (CM) to the ECP: ${\bf r} = {\bf a}_1 - {\bf q}$, ${\bf 0}$ is a $3\times 1$ vector with each entry equals to zero. 
\subsection{Kinematic map}
The kinematic map below gives the relationship between the the generalized velocity $\bm{\nu}$ and the time derivative of the position and orientation, $\dot{{\bf q}}$. For unit quaternion representation of rotation, ${\bf G}$ is a $6 \times 7$ matrix. 
\begin{equation}
\dot{\bf q} = {\bf G}({\bf q}) \bm{\nu}
\end{equation}

\subsection{Friction Model}
Our friction model is based on the maximum power dissipation principle and generalized Coulomb's friction law, which has been previously proposed in the literature for point contact~\cite{Moreau1988}. The maximum power dissipation principle states that among all the possible contact wrenches (i.e., forces and moments) 
%that lie within the friction ellipsoid, 
the wrench that maximize the power dissipation at the contact are selected. 

For non-point contact, we will use a generalization of the maximum power dissipation principle, where, we select contact wrenches and contact velocities that maximize the power dissipation over the entire contact patch. In~\cite{XieB+2019}, we have shown that the problem formulation using the power loss over the whole contact patch can be reduced to the friction model for point contact with the ECP as the chosen point. Mathematically, the power dissipated over the entire surface, $P_c$ is given by:
\begin{equation}
   P_c = - (v_t \lambda_t + v_o\lambda_o + v_r \lambda_r)
\end{equation}
 where $v_t ={\bf W}^{T}_{t}
\bm{\nu} $ and $v_o={\bf W}^{T}_{o}
\bm{\nu}$ are the components of tangential velocities at the ECP. Similarly, the angular velocity about contact normal $v_r ={\bf W}^{T}_{r}
\bm{\nu} $. $\lambda_t$, $\lambda_o$ are the magnitudes net tangential forces at the ECP and $\lambda_r$ is the magnitude of net moment about the axis normal to the contact patch and passing through the ECP.

For specifying a friction model, we also need a law or relationship that bounds the magnitude of the friction forces and moments in terms of the magnitude of the normal force~\cite{Goyal1991}. Here, we use an ellipsoidal model for bounding the magnitude of tangential friction force and friction moment. This friction model has been previously proposed in the literature~\cite{Goyal1991,Moreau1988,XieC16,NilanjanChakraborty2007} and has some experimental justification~\cite{Howe1996}.
Thus, the contact wrench is the solution of the following optimization problem: 
\begin{equation}
\begin{aligned}
\label{equation:friction}
\argmax_{\lambda_t, \lambda_o, \lambda_r} \quad -(v_t \lambda_t + v_o\lambda_o + v_r \lambda_r)\\
{\rm s.t.} \quad \left(\frac{\lambda_t}{e_t}\right)^2 + \left(\frac{\lambda_o}{e_o}\right)^2+\left(\frac{\lambda_r}{e_r}\right)^2 - \mu^2 \lambda_n^2 \le 0
\end{aligned}
\end{equation}
where the magnitude of contact force and moment at the ECP, namely, $\lambda_t$, $\lambda_o$, and $\lambda_r$ are the optimization variables. The parameters, $e_t$, $e_o$, and $e_r$ are positive constants defining the friction ellipsoid and $\mu$ is the coefficient of friction at the contact~\cite{Howe1996,Trinkle1997}. Thus, we can use the contact wrench at the ECP to model the effect of entire distributed contact patch. Note that, distinct from~\cite{XieC16}, the contact patch may not be convex. 

\subsection{Time-stepping Formulation}
We use a velocity-level formulation and an Euler time-stepping scheme to discretize the above system of equations. Let $t_u$ denote the current time and $h$ be the duration of the time step, the superscript $u$ represents the beginning of the current time and the superscript $u+1$ represents the end of the current time. Using $\dot{\bm{\nu}} \approx ( {\bm{\nu}}^{u+1} -{\bm{\nu}}^{u} )/h$, $\dot{\bf q} \approx( {\bf q}^{u+1} -{\bf q}^{u} )/h$ and writing forces as impulses ( $p_{(.)}=h\lambda_{(.)}$), we  discretize Newton-Euler equations and kinematic map:
\begin{equation}
\begin{aligned}
\label{equ:discrete_NE}
0 = &
-{\bf M}^{u+1}({\bm{\nu}}^{u+1} - {\bm{\nu}}^{u})+ {\bf W}_{n}p^{u+1}+ {\bf W}_{t}p^{u+1}\\
&+{\bf W}_{o}p^{u+1}_{o} +{\bf W}_{r}p^{u+1}_{r} +{\bf p}^{u}_{app}+{\bf p}^{u}_{vp}
\end{aligned}
\end{equation}
\begin{equation}
\label{equ:discrete_KE}
0 =-{\bf q}^{u+1}+{\bf q}^{u}+ h{\bf G}({\bf q}^u)\bm{\nu}^{u+1}
\end{equation}

Using the Fritz-John optimality conditions of Equation~\eqref{equation:friction}, we can write~\cite{trinkle2001dynamic}:
\begin{equation}
\begin{aligned}
\label{eq:friction}
0&=
e^{2}_{t}\mu p^{u+1}_{n} 
{\bf W}^{T}_{t}
\bm{\nu}^{u+1}+
p^{u+1}_{t}\sigma^{u+1} \\
0&=
e^{2}_{o}\mu p^{u+1}_{n}  
{\bf W}^{T}_{o}
\bm{\nu}^{u+1}+p^{u+1}_{o}\sigma^{u+1} \\
0&=
e^{2}_{r}\mu p^{u+1}_{n}{\bf W}^{T}_{r}
\bm{\nu}^{u+1}+p^{u+1}_{r}\sigma^{u+1}\\
0& \le (\mu p^{u+1}_n)^2- (\frac{p^{u+1}_{t}}{e_{t}})^2- (\frac{p^{u+1}_{o}}{e_{o}})^2- (\frac{p^{u+1}_{r}}{e_{r}})^2 \perp \sigma^{u+1} \ge 0
\end{aligned}
\end{equation}
where $\sigma$ is a Lagrange multiplier corresponding to the inequality constraint in~\eqref{equation:friction}. Note that ${\bf W}_{t}, {\bf W}_{n}, {\bf W}_{o}, {\bf W}_{r}$ in Equations~\eqref{eq:friction} are dependent on ECPs at the end of time step $u+1$. Therefore, our discrete-time model is a {\em geometrically implicit} model.

\section{Modeling Planar Non-convex Patch Contact}
\label{se:Model_nonconvex}
In this section, we will present our method for modeling a planar non-convex contact patch. Although, we will present the equations here in a more general manner, for concreteness, one can think  that one object is a non-convex object and the other object is a plane (or a face of a polyhedron much larger than the non-convex object). This is the scenario where planar non-convex contact patch is easy to visualize and this situation is quite prevalent in robotics.
%model for computing the closest point between a non-convex object
Let $F$ and $G$ be the two objects, where, without loss of generality, the object $F$ is the non-convex object and $G$ is the convex object.
When two objects $F$ and $G$ have planar contact, the planar contact patch $\mathcal{S}$ is a non-empty finite subset of line or  plane. We will use the convex hull of object $F$, denoted by $Conv(F)$ to model the non-convex object $F$ (this will be justified later in the section). We will now present the contact constraints for non-penetration of rigid bodies.

\subsection{Non-penetration constraints}
In complementarity-based formulation of dynamics, the contact constraint for a potential contact is written as 
\begin{equation} \label{equation:normal contact}
0\le \lambda_{n} \perp \psi_{n}(
\bf{q},t) \ge 0
\end{equation}
where $\psi_{n}(\bf{q},t)$ is the {\em gap function or distance function} for the contact with the property $\psi_{n}(\bf{q},t) > 0$ for separation, $\psi_{n}(\bf{q},t) = 0$ for touching and $\psi_{n}(\bf{q},t) < 0$ for interpenetration. The complementarity function models the physical fact that the contact force magnitude is positive when the objects are in contact (i.e., distance function is zero) and the contact force magnitude is zero when the distance function is greater than zero. When both distance function and contact force is equal to zero, it implies grazing contact with tangential velocity (i.e., no normal component of relative velocity towards the surfaces at the contact point). We will also call the constraints in~\eqref{equation:normal contact} as the {\em non-penetration constraints}, since they ensure the constraints that solids cannot penetrate each other (i.e., $\psi_n(\bf{q}, t) \geq 0$). 

Note that there is usually no closed form expression for $\psi_{n}(\bf{q},t)$. Thus, in a discrete-time framework, it is usually hard to ensure satisfaction of the complementarity constraints at the end of the time step.  A collision detection module provides information about the closest (contact) points and the normal to the object surfaces at these points, which is used to construct a first order approximation of the distance function. Thus, only a first order approximation of the non-penetration constraints are satisfied at the end of the time step. This can lead to inaccuracies in motion prediction, even for point contact because of phantom collisions or penetration between the objects~\cite{NilanjanChakraborty2007}. For non-point contact, there can be an uncountably many number of contact points and thus the collision detection problem becomes ill-posed.

In~\cite{NilanjanChakraborty2007}, we presented a method for incorporating the geometry of the contacting objects so that Equation~\eqref{equation:normal contact} is satisfied exactly at the end of the time step and the contact points at the end of the time step are obtained. In~\cite{XieC16}, we showed that when the contact patch is a convex contact patch, this method actually computes the ECP along with the net contact wrench acting at the ECP to ensure that the non-penetration constraints are satisfied at the end of the time step. We will now show that the contact constraints presented below allows us to compute the ECP of a non-convex contact patch as well as the contact wrench (that ensures that the non-penetration constraints are satisfied at the end of the time step) as part of the numerical integration of the equations of motion.

We assume that the convex hull of $F$, i.e., $Conv(F)$, and $G$ are
described by the intersecting convex inequalities $f_{i}(\bf{x}) \le 0, i = 1,...,m$, and $g_{j}(\bf{x}) \le 0, j = m+1,...,n$ respectively. Note that each individual convex constraint $f_{i}(\bf{x}) = 0$  describes the boundary of the convex hull. We also assume that the object $F$ is described by an intersection of inequalities, not necessarily convex. Single point contact, multi-point contact, and convex patch contact are all special cases of the contact that we are considering. Let $\bf{a}_1$ and $\bf{a}_2$ be the pair of equivalent contact points for $Conv(F)$ and $G$ respectively. For single point contact, $\bf{a}_1$ and $\bf{a}_2$ are the contact points on the two objects. Note that, in general, $\bf{a}_1$ may not be a point in $F$.

We will now rewrite the contact condition (Equation~\eqref{equation:normal contact}) in terms of the convex inequalities describing the objects, and combine it with an optimization problem to find the closest points.
Note that for any object that is described by a collection of inequalities $f_i(\bf{x}) \leq 0$, $i = 1, \dots, m$, then for any point $\bf{x}$, the point lies inside the object when $f_i(\bf{x}) < 0$, $\forall i$, on the boundary of object when $f_i(\bf{x}) = 0$ for some $i$ and $f_j(\bf{x}) \leq 0$, $j = 1, \dots, m$, $j \neq i$, and outside the object when $f_i(\bf{x}) > 0$ for some $i$. Thus, the contact condition (Equation~\eqref{equation:normal contact}) can be rewritten as one of the following two complementarity constraints~\cite{NilanjanChakraborty2007} by either using the distance function $\psi_n(\bf{q},t)$ as $\mathop{max}_{i=1,...,m} f_{i}(\bf{a}_{2}) \ge 0$ or $\mathop{max}_{j=m+1,...,n}g_{j}(\bf{a}_{1}) \ge 0$.
\begin{align}
\label{equation:contact_multiple_comp_1}
&0 \le \lambda_{n} \perp \mathop{max}_{i=1,...,m} f_{i}(\bf{a}_{2}) \ge 0\\
\label{equation:contact_multiple_comp_2}
&0 \le \lambda_{n} \perp \mathop{max}_{j=m+1,...,n}g_{j}(\bf{a}_{1}) \ge 0
\end{align}
In Equation~\eqref{equation:contact_multiple_comp_1}, if $\mathop{max}_{i=1,...,m} f_{i}(\bf{a}_{2}) > 0$ then the closest point on object $G$ to the convex hull of $F$ lies outside the set $Conv(F)$ and hence the object $F$. Thus, the objects are not in contact and consequently, $\lambda_n = 0$, (i.e., there is no contact force). If $\mathop{max}_{i=1,...,m} f_{i}(\bf{a}_{2}) = 0$, then $Conv(F)$ and $G$ are in contact and $\lambda_n > 0$. In this case, if the object $G$ is a flat plane, then we can conclude that object $F$ and $G$ are in contact, which would imply that the contact force magnitude $\lambda_n > 0$. If $G$ is not a flat plane, it does not necessarily imply a contact between $F$ and $G$. There are three cases that may arise (a) there is contact between $F$ and $G$ and the contact patch is planar (b) there is contact between $F$ and $G$ and the contact patch is non-planar (c) there is no contact between $F$ and $G$. For case (a), we can use the above equation as it is. For case (b), our method does not apply and we will not consider it further. For case (c), we have to perform additional computational checks.  We discuss both case (a) and case (c) below in more detail after we present the equations for computing the closest points $\bf{a}_1$ and $\bf{a}_2$. 

The closest points $\bf{a}_{1}$ and $\bf{a}_{2}$ are given by a solution to the following minimization problem for computing the distance between convex hull of $F$ and $G$:
\begin{equation}
\label{equation:optimazation}
(\bf{a}_{1},\bf{a}_{2}) = \arg \min_{\bf{\zeta}_1,\bf{\zeta}_2}\{ \|\bf{\zeta}_1-\bf{\zeta}_2 \| : \ f_{i}(\bf{\zeta}_1) \le 0,\ g_{j}(\bf{\zeta}_2) \le 0 \}
\end{equation}
As shown in~\cite{NilanjanChakraborty2007}, based on a modification of the KKT conditions, we can show that the ECPs need to satisfy the algebraic and complementarity constraints given below to solve the optimization problem above (Equation~\eqref{equation:optimazation}). We refer the readers to~\cite{NilanjanChakraborty2007} for the derivation of these equations. 
\begin{align}
\label{equation:re_contact_multiple_1}
&\bf{a}_{1}-\bf{a}_{2} = -l_{k}\nabla\mathcal{C}(\bf{F}_i,\bf{a}_{1})\\
\label{equation:re_contact_multiple_2}
&\nabla\mathcal{C}(\bf{F}_i,\bf{a}_{1})= -\sum_{j = m+1}^{n} l_{j} \nabla g_{j} (\bf{a}_{2})\\
\label{equation:re_contact_multiple_3}
&0 \le l_{i} \perp -f_{i}(\bf{a}_{1}) \ge 0 \quad i = 1,..,m,\\
\label{equation:re_contact_multiple_4}
&0 \le l_{j} \perp -g_{j}(\bf{a}_{2}) \ge 0 \quad j = m+1,...,n.
\end{align}
where $\nabla\mathcal{C}(\bf{F}_i,\bf{a}_{1}) = \nabla f_{k}(\bf{a}_{1})+\sum_{i\neq k}^{m} l_{i}\nabla f_{i}(\bf{a}_{1})$, $k$ represents the index of any one of the active constraints (i.e., the surface on which the ECP $\bf{a}_{1}, \bf{a}_{2}$ lies). We will also need an additional complementarity constraint (either Equation~\eqref{equation:contact_multiple_comp_1} or Equation~\eqref{equation:contact_multiple_comp_2}) to prevent penetration:
\begin{align}
\label{equation:re_contact_multiple_5}
0 \le \lambda_{n} \perp \mathop{max}_{j=m+1,...,n} g_{j}(\bf{a}_{1}) \ge 0
\end{align}
Equations~\eqref{equation:re_contact_multiple_1}$\sim$~\eqref{equation:re_contact_multiple_5} together gives the constraints that the equivalent contact points $\bf{a}_{1}$ and $\bf{a}_{2}$ should satisfy for ensuring no penetration between the objects. We prove this formally in Proposition $2$. 

As discussed in Section~\ref{sec:ECP}, the ECP lies in the convex hull of the contact patch. However, we do not know the contact patch at the end of the time step {\em a priori}, so it is not possible to compute the convex hull of the contact patch {\em a priori}. In the collision constraints to compute the ECP, we have used the convex hull of the object $F$ to formulate the equations. We prove below that when there is contact, by using the convex hull of $F$, the computed ECP lies within the convex hull of the contact patch. Thus, we do not need any {\em a priori} knowledge about the contact patch.

%Note that when objects are separate, the equivalent contact points $\bm{a}_{1}$ and $\bm{a}_{2}$ are solved as pair of closest points on the convex hull of $F$ and $G$. However, this does not lead to any inaccuracies since separation of $Conv(F)$ from $G$ implies separation of $F$ and $G$ and vice-versa.  When objects have contact, $\bm{a}_{1}$ and $\bm{a}_{2}$ are solved as touching solution which prevents penetration between objects.

%We first prove that the use of the convex hull ensures that the ECP that we compute is within the convex hull of the contact patch. 

\begin{definition}
Let $\bf{x}$ be a point that lies on the boundary of a compact set $F$. Let $\mathbb{I}$ be the index set of active constraints for ${\bf x}$, i.e., $\mathbb{I} = \{i | f_i(\bf{x}) = 0, \ i = 1, 2, \dots, n \}$.  The normal cone to $F$ at ${\bf x}$, denoted by $\mathcal{C}(F,\bm{x})$, consists of all vectors in the conic hull of the normals to the surfaces (at ${\bf x}$) represented by the active constraints. Mathematically,
$$\mathcal{C}(F,\bf{x}) = \{ \bf{y} \vert  \bf{y} = \sum_{i \in \mathbb{I}} \beta_i \nabla f_i(\bf{x}), \beta_i\ge 0\}$$.
%where $\beta_i$ are non-negative constants.
\end{definition}

\begin{definition}
\label{def:hyper}
Let $F$ be a compact convex set and let ${\bf x}_0$ be a point that lies on the boundary of $F$. Let $\mathcal{C}(F,\bf{x}_0)$ be the normal cone of $F$ at ${\bf x}_0$. The supporting plane of $F$ at ${\bf x}_0$ is a plane passing through ${\bf x}_0$ such that all points in $F$ lie on the same side of the plane. In general, there are infinitely many possible supporting planes at a point. In particular any plane
$ \mathcal{H}({\bf x}) = \{ {\bf x}|  \bm{\alpha}^T({\bf x} - {\bf x}_0) = 0\}$ 
% Such that:
%$$\mathcal{H}(\bf{x}_o)= 0, \forall \bf{x}_o\in \partial F$$ %$$\mathcal{H}(\bf{x})\le \mathcal{H}(\bf{x}_o), \forall \bf{x}\in F$$ 
 where ${\bm{\alpha}} \in \mathcal{C}(F,\bf{x}_0)$ is a supporting plane to $F$ at ${\bf x}_0$.
\end{definition}
\comment{
\begin{definition}
The touching solution between two objects $F$ and $G$ is for ECPs $\bf{a}_1$ and $\bf{a}_2$  satisfying:
\begin{enumerate}
\item The points $\bf{a}_1$ and $\bf{a}_2$ that satisfy Equations~\eqref{equation:re_contact_multiple_1} to \eqref{equation:re_contact_multiple_5} lie on the boundary of objects $F$ and $G$ respectively. 
\item The objects can not intersect with other.
\end{enumerate}
\end{definition}
}

\begin{proposition}
\label{pro:1}
Suppose the contact patch between object $F$ and object $G$ lies on a plane, i.e., the contact patch is planar. Then, by using the convex hull of the object $F$ to formulate the contact constraints, we ensure that we compute the ECP within the convex hull of the contact patch.
\end{proposition}
\begin{proof}
%There are three steps in the proof. First we have to prove that the convex hull description captures the set of all points of F that can have planar contact. The set of points on F that can potentially have contact with a plane are the set of extreme points on the boundary of F. The subsets of extreme points that lie on a plane are all the possible contact regions with a plane. If any one of these subsets are in contact with a plane

Let $\partial F$ and $\partial Conv(F)$ be the boundaries of the object $F$ and the convex hull of object $F$ respectively. A point $\bf{x} \in$ $\partial F$ is called an extreme point of $F$, if and only if there exists a plane passing through $\bf{x}$, such that all points in $F$ lie on one side of $\mathcal{H}$. Let $\mathcal{E}(F)$ be the set of extreme points of $F$. For a convex set, the set of its extreme points are same as its boundary. Thus, the set of extreme points of $Conv(F)$ is $\partial Conv(F)$.  Furthermore, from the properties of convex hulls, $Conv(F)$ contains the set of all extreme points of $F$, i.e., $\mathcal{E}(F) \subseteq \partial Conv(F)$. For a non-convex object contacting with a plane, the set of extreme points are the only points that can potentially contact the plane. Therefore, using the convex hull description ensures that we are capturing the set of all boundary points of $F$ that can be in contact.

All the possible planar contact regions are subsets of $\mathcal{E}(F)$ that lie on a plane. Let $C_1 \subset \mathcal{E}(F)$ be a planar contact patch between object $F$ and $G$. Let $C_2 \subset \partial Conv(F)$ be a planar contact patch between the convex hull of object $F$ and $G$. Note that $C_2$ is always a convex set, since both $Conv(F)$ and $G$ are convex sets. Our goal is to prove that $C_2$ is the convex hull of $C_1$, i.e., $C_2 = Conv(C_1)$. 

Let $\mathcal{H}(\bf{x}) = 0$ be the plane of the contact region. This plane is also a supporting plane for $F$ and $G$. With abuse of notation, $\mathcal{H} = \{{\bf x}| \mathcal{H}({\bf x}) = 0 \}$, i.e., $\mathcal{H}$ is the set of all points lying on the supporting plane.  Now, we can write $C_1 = \mathcal{E}(F) \cap \mathcal{H}$ and $C_2 =  \partial Conv(F) \cap \mathcal{H}$.  Since, $\mathcal{E}(F) \subseteq \partial Conv(F)$, we can conclude that $C_1 \subseteq C_2$. In words, the planar contact patch on object $F$ is a subset of the planar contact patch formed with the convex hull of $F$.

If $C_1 = C_2$, then $C_1$ is a convex set and thus $Conv(C_1) = C_1 = C_2$. If $C_1 \subset C_2$, i.e., there are points in $C_2$ that do not lie in $C_1$, we have to show that these points do not belong to $F$, i.e., $\{C_2 \setminus C_1\} \cap F = \phi$, where $\phi$ denotes the empty set. Since $\mathcal{E}(F)$ are the only points where $F$ can intersect $\mathcal{H}$, therefore it suffices to show that $\{C_2 \setminus C_1\} \cap \mathcal{E}(F) = \phi$.  We will prove this by contradiction. Assume that ${\bf y} \in C_2$, ${\bf y} \notin C_1$ and ${\bf y} \in \mathcal{E}(F)$. Since ${\bf y} \in C_2$, it implies that ${\bf y} \in \mathcal{H}$. Since $C_1 = \mathcal{E}(F) \cap \mathcal{H}$ and ${\bf y} \in \mathcal{H}$, ${\bf y} \notin C_1$ implies ${\bf y} \notin \mathcal{E}(F)$. But by assumption, ${\bf y} \in \mathcal{E}(F)$, which leads to a contradiction. Thus, $\{C_2 \setminus C_1\} \cap \mathcal{E}(F) = \phi$.

Since $C_1 \subset C_2$ and points in $C_2$ that do not belong to $C_1$ does not belong to $F$, therefore $C_2 = Conv(C_1)$. Furthermore, in Equations~\eqref{equation:re_contact_multiple_1}$\sim$~\eqref{equation:re_contact_multiple_5}, the ECP will lie in $C_2$. Therefore, for planar contact, the ECP computed our contact constraints in Equations~\eqref{equation:re_contact_multiple_1}$\sim$~\eqref{equation:re_contact_multiple_5} based on the convex hull of $F$ will lie in the convex hull of the contact patch of object $F$ with $G$.

% Given a plane $\mathcal{H}(\bf{x})$ defined at a boundary point $\bf{x} \in \partial F$, the $\mathcal{H}(\bf{x})$ is  the supporting hyperplane of $F$, if and only if $\bf{x} \in \partial Conv(F)$. In other words, the contact points, that lie on the contact patch $C_1$ between object $F$ and plane, should also lie on the contact patch $C_2$ between $Conv(F)$ and plane.

%(2) The contact patch $C_2$ is the convex hull of contact patch $C_1$.

% So when we are solving for the ECP, it will be in the convex region defined by the active constraints which is essentially the convex hull of the set of contacting points. 
\end{proof}

\begin{remark}
As stated earlier, when the distance between the convex hull of $F$ and $G$ is zero, but the distance between $F$ and $G$ is non-zero, we need to perform additional computational checks. At the end of each time step we can check to see if the closest point on object $G$, i.e., ${\bf a}_2$ is outside or on the object $F$. As $F$ is described by a set of inequalities, if any one of the inequalities evaluated at ${\bf a}_2$ is positive, it would imply that the object $G$ and $F$ are separate. Note that when object $G$ is a plane, this additional step is not necessary as the distance between the convex hull of $F$ and a plane is zero implies that the distance between $F$ and the plane is zero.
\end{remark}

\begin{proposition}
\label{pro:2}
{When using Equations~\eqref{equation:re_contact_multiple_1} to~\eqref{equation:re_contact_multiple_5} to model the contact between convex hulls for two objects, we get the solution for ECPs as the closest points on the boundary of convex hulls respectively when objects are separate. When objects have planar contact, we will get touching solution which prevents penetration. }
\end{proposition}
\begin{proof}
%Because of lack of space, we do not provide the full proof here. The proof essentially follows from the arguments of the proof shown in~\cite{NilanjanChakraborty2007} and~\cite{XieC16}, with minor modification to consider the convex hull of $F$ instead of $F$.
The proof idea follows from the arguments of the proof shown in~\cite{NilanjanChakraborty2007} and~\cite{XieC16}, with modifications done to consider the convex hull of $F$ instead of $F$.

When objects are separate, Equations~\eqref{equation:re_contact_multiple_1} $\sim$~\eqref{equation:re_contact_multiple_5} will give us the solution for $\bf{a}_1$ and $\bf{a}_2$ as the closet points on the boundary of $Conv(F)$ and $G$ respectively. The proof is same as in~\cite{NilanjanChakraborty2007}. 

When the distance between two objects is zero, the modified KKT conditions~\eqref{equation:re_contact_multiple_1} to~\eqref{equation:re_contact_multiple_4} will give us the optimal solution for the minimization problem in Equation~\eqref{equation:optimazation}, i.e., ${\bf a}_1 = {\bf a}_2$. Furthermore, Equations~\eqref{equation:re_contact_multiple_1} to~\eqref{equation:re_contact_multiple_4} and Equation~\eqref{equation:re_contact_multiple_5} together give us the solution for ${\bf a}_1$ and ${\bf a}_2$ as the touching solution for planar contact, i.e.,:
\begin{enumerate}
\item The points ${\bf a}_1$ and ${\bf a}_2$ that satisfy Equations~\eqref{equation:re_contact_multiple_1} to \eqref{equation:re_contact_multiple_5} lie on the boundary of the convex hull $Conv(F)$ and $G$ respectively. 
\item The interior of the set $Conv(F)$ cannot intersect with the interior of the set $G$.
\end{enumerate}

We prove the first part by contradiction. If ${\bf a}_1$ lies within the interior of $Conv(F)$, then from Equation~\eqref{equation:re_contact_multiple_3}, $f_i({\bf a}_1) < 0,\  l_i = 0, \ \forall i = 1,...,m$. From Equation~\eqref{equation:re_contact_multiple_1}, ${\bf a}_1 = {\bf a}_2$, thus $f_i({\bf a}_2) < 0, \ \forall i = 1,...,m$, which contradicts with Equation~\eqref{equation:re_contact_multiple_5}. Thus ${\bf a}_1$ has to lie on the boundary of $Conv(F)$. If ${\bf a}_2$ lies within object $G$, from Equation~\eqref{equation:re_contact_multiple_4}, $g_j({\bf a}_2) < 0$, $l_j = 0,  \ \forall j = m+1,...,n$. Thus, $\sum_{j = m+1}^n l_j \nabla g_j ({\bf a}_2) = 0$. Since the left hand side of Equation~\eqref{equation:re_contact_multiple_2} is nonzero, this leads to a contradiction. Thus ${\bf a}_2$ lies on the boundary of object $G$. 

We will now prove that the interior of $Conv(F)$ and $G$ are disjoint. We prove it based on the supporting hyperplane theorem. Let $\mathcal{H}$ be the supporting plane to $Conv(F)$ at the point ${\bf{a}_1} \in \partial Conv(F)$, where the normal ${\bm \alpha} \in \mathcal{C}(\bf{F}_i,\bf{a}_{1}) $. 
%Since $\bf{a}_{1}$ lies on $\mathcal{H}$, which can be mathematically written as ${\bm \alpha}^T\bf{a}_{1}+{\beta} = 0$. Therefore, the constant ${\beta} = - {\bm \alpha}^T\bf{a}_{1}$, and 
The supporting plane is given by $\mathcal{H} = \{{\bf x} | \bm{ \alpha}^T({\bf x}-{\bf a}_{1})= 0 \}$. Since the plane $\mathcal{H}$ supports $ Conv(F)$ at ${\bf a}_{1}$,  for all points ${\bf x} \in  Conv(F)$, the affine function $\bm{ \alpha}^T({\bf x}-{\bf a}_{1})\le 0$. In other words, the halfspace $\{{\bf x} | \bm{\alpha}^T({\bf x}-{\bf a}_{1})\le 0 \}$ contains $Conv(F)$.  Now we need to prove that the halfspace $\{{\bf x} | \bm{\alpha}^T({\bf x}-{\bf a}_{1}) \ge 0 \}$ contains object $G$, which would imply that objects $Conv(F)$ and $G$ can be separated by $\mathcal{H}$. For point ${\bf a}_2 \in \partial G$, since ${\bf a}_1 = {\bf a}_2$, ${\bf a}_2$ lies in $\mathcal{H}$. For other points ${\bf y} \in \{G \setminus {\bf a}_2\}$, we have $\bm{\alpha}^T({\bf y}-{\bf a}_{1}) = \bm{ \alpha}^T({\bf y}- {\bf a}_{2}+{\bf a}_{2} -{\bf a}_{1}) = \bm{ \alpha}^T({\bf y}-{\bf a}_{2})$. From Equation~\eqref{equation:re_contact_multiple_2}, the direction of normal $\bm{\alpha}$ is  opposite to the normal cone of $G$ at ${\bf a}_2$. Since object $G$ is convex, the projection of the vector ${\bf y}-{\bf a}_{2}$ onto the normal cone at ${\bf a}_2$ is always non-positive. Therefore, the function ${\bm \alpha}^T({\bf y}-{\bf a}_{2})$ is always non-negative. Thus, the halfspace $\{{\bf x} | {\bm \alpha}^T({\bf x}-\bf{a}_{1})\ge 0 \}$ contains object $G$. Thus, we can conclude that the interior of $Conv(F)$ and $G$ are disjoint.

\end{proof}

% end comment

\subsection{Summary of the discrete-time dynamic model}
As stated earlier, our dynamic model is composed of (a) Newton-Euler equations (Equation~\eqref{equ:discrete_NE}), (b) kinematic map between the rigid body generalized velocity and the rate of change of the parameters for representing position and orientation (Equation~\eqref{equ:discrete_KE}), (c) contact model  which gives the constraints that the equivalent contact points ${\bf a}_{1}$ and ${\bf a}_{2}$ should satisfy for ensuring no penetration between the objects (Equations~\eqref{equation:re_contact_multiple_1}$\sim$~\eqref{equation:re_contact_multiple_5}). (d) friction model  which gives the constraints that contact wrenches should satisfy (Equation~\eqref{eq:friction}). Thus, we have a coupled system of algebraic and complementarity equations (mixed nonlinear complementarity problem) that we have to solve. The vector of unknowns, ${\bf z} = [{\bf z}_u;{\bf z}_v]$ where the vector for unknowns of equality constraints is ${\bf  z}_u = [\bm{\nu}^{u+1}; {\bf a}_1^{u+1}; {\bf a}_2^{u+1}; p_t^{u+1}; p_o^{u+1}; p_r^{u+1}]$ and the vector for unknowns of complementary constraints is ${\bf z}_v = [{\bf l}_1;{\bf l}_2;p_n^{u+1};\sigma^{u+1}]$.
The equality constraints in the mixed NCP are:
\begin{equation}
\begin{aligned}
0 &= -{\bf M}^{u} {\bm{\nu}}^{u+1} +
{\bf M}^{u}{\bm{\nu}}^{u}+{\bf W}_{n}^{u+1}p^{u+1}_{n}+{\bf W}_{t}^{u+1}p^{u+1}_{t} \\&
+{\bf W}_{o}^{u+1}p^{u+1}_{o}+{\bf W}_{r}^{u+1}p^{u+1}_{r}+ {\bf p}^{u}_{app}+{\bf p}^{u}_{vp}\\
0& = {\bf a}^{u+1}_1-{\bf a}^{u+1}_2+l^{u+1}_{k_1}(\nabla f_{k_1}({\bf a}^{u+1}_1)+\sum_{i = 1,i\neq k_1}^m l^{u+1}_i\nabla f_i({\bf a}^{u+1}_1))\\
0&= \nabla f_{k_1}({\bf a}^{u+1}_1)+\sum_{i = 1,i\neq k_1}^m l^{u+1}_i\nabla f_i({\bf a}^{u+1}_1)+\sum_{j = m+1}^n l^{u+1}_j \nabla g_j ({\bf a}^{u+1}_2)\\
0&=
e^{2}_{t}\mu p^{u+1}_{n} 
({\bf W}^{T}_{t})^{u+1}
\bm{\nu}^{u+1}+
p^{u+1}_{t}\sigma^{u+1}\\
0&=
e^{2}_{o}\mu p^{u+1}_{n}  
({\bf W}^{T}_{o})^{u+1}
\bm{\nu}^{u+1}+p^{u+1}_{o}\sigma^{u+1}\\
0&=
e^{2}_{r}\mu p^{u+1}_{n}({\bf W}^{T}_{r})^{u+1}
\bm{\nu}^{u+1}+p^{u+1}_{r}\sigma^{u+1}
\end{aligned}
\end{equation}

The complementary constraints for $\bm{z}_v$ are:
\begin{equation}
\begin{aligned}
0 \le \left[\begin{matrix}
{\bf l}_1\\
{\bf l}_2\\
p^{u+1}_n\\
\sigma^{u+1}
\end{matrix}\right]\perp 
\left[\begin{matrix}
-{\bf f}({\bf a}^{u+1}_1)\\
-{\bf g}({\bf a}^{u+1}_2)\\
\max\limits_{i=1,...,m} f_i({\bf a}^{u+1}_2) \\
\zeta
\end{matrix}\right] \ge 0
\end{aligned}    
\end{equation}
where $\zeta =(\mu p^{u+1}_n)^2- (p^{u+1}_{t}/e_{t})^2- (p^{u+1}_{o}/e_{o})^2- (p^{u+1}_{r}/e_{r})^2 $.

\comment{
The solution of the MNCP for each time step gives us the linear and angular velocity ${\bm{\nu}}^{u+1}$. It also returns us contact impulses $p_{n_i}^{u+1}, p_{t_i}^{u+1}, p_{o_i}^{u+1}, p_{r_i}^{u+1}$, and equivalent contact points ${\bf a}_{1_i}^{u+1}, {\bf a}_{2_i}^{u+1}$ for $i$th contact $(i = 1,...,n_c)$, at the end of the time-step. The position and orientation of the object is obtained from Equation~\eqref{equ:discrete_KE}, after we obtain ${\bm{\nu}}^{u+1}$. 
}

%Note that our geometrically implicit scheme is not valid when planar contact between objects is a non-convex set.
\comment{
When objects have the planar contact, it can be viewed as contact between one or multiple pairs of bodies or parts that forming the objects. When the planar contact between objects is a non-convex set, there exists one or multiple bodies that is non-convex. Thus, the minimization problem in Equation~\eqref{equation:optimazation} would not be convex optimization problem anymore.}

\section{Simulation Results}
\label{se:simu}
\begin{figure*}[!htp]%
%\centering
\begin{subfigure}{1.2\columnwidth}
\includegraphics[width=\columnwidth]{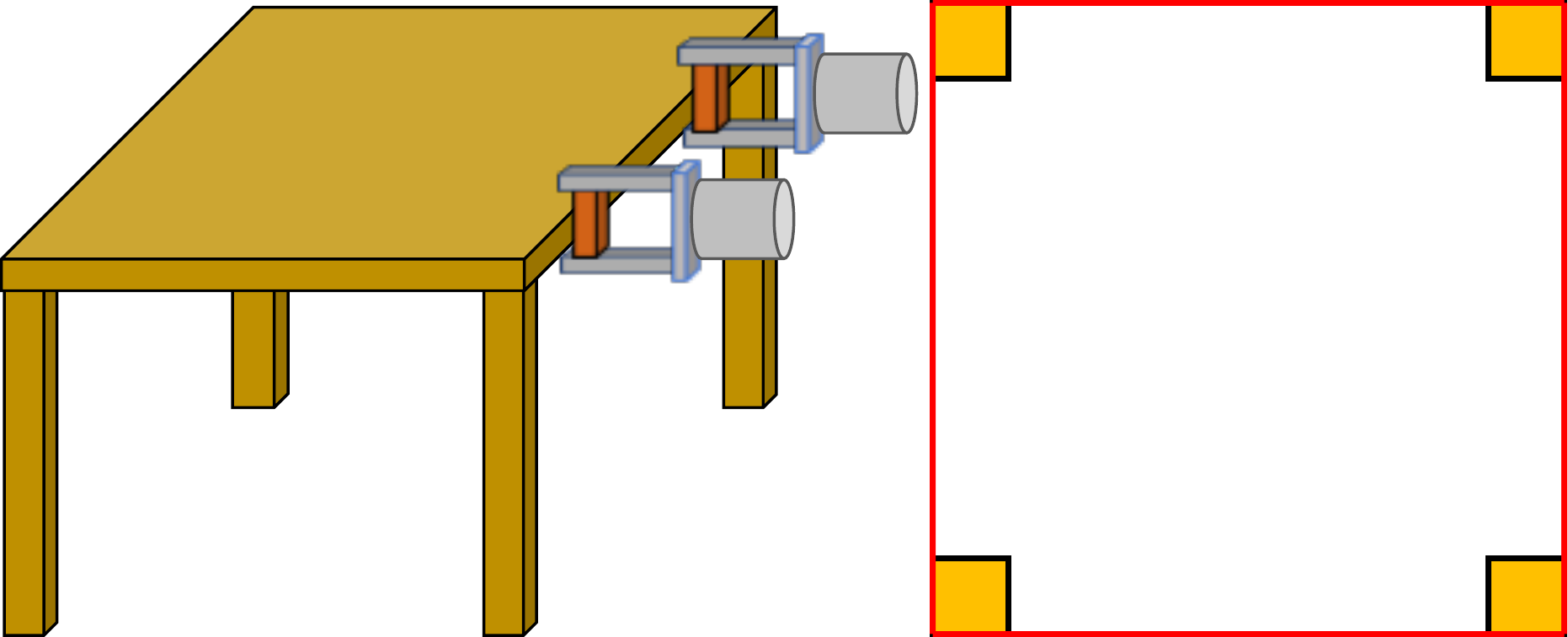}%
\caption{}
\label{figure:ex1_1} 
\end{subfigure}\hfill
\begin{subfigure}{0.8\columnwidth}
\includegraphics[width=\columnwidth]{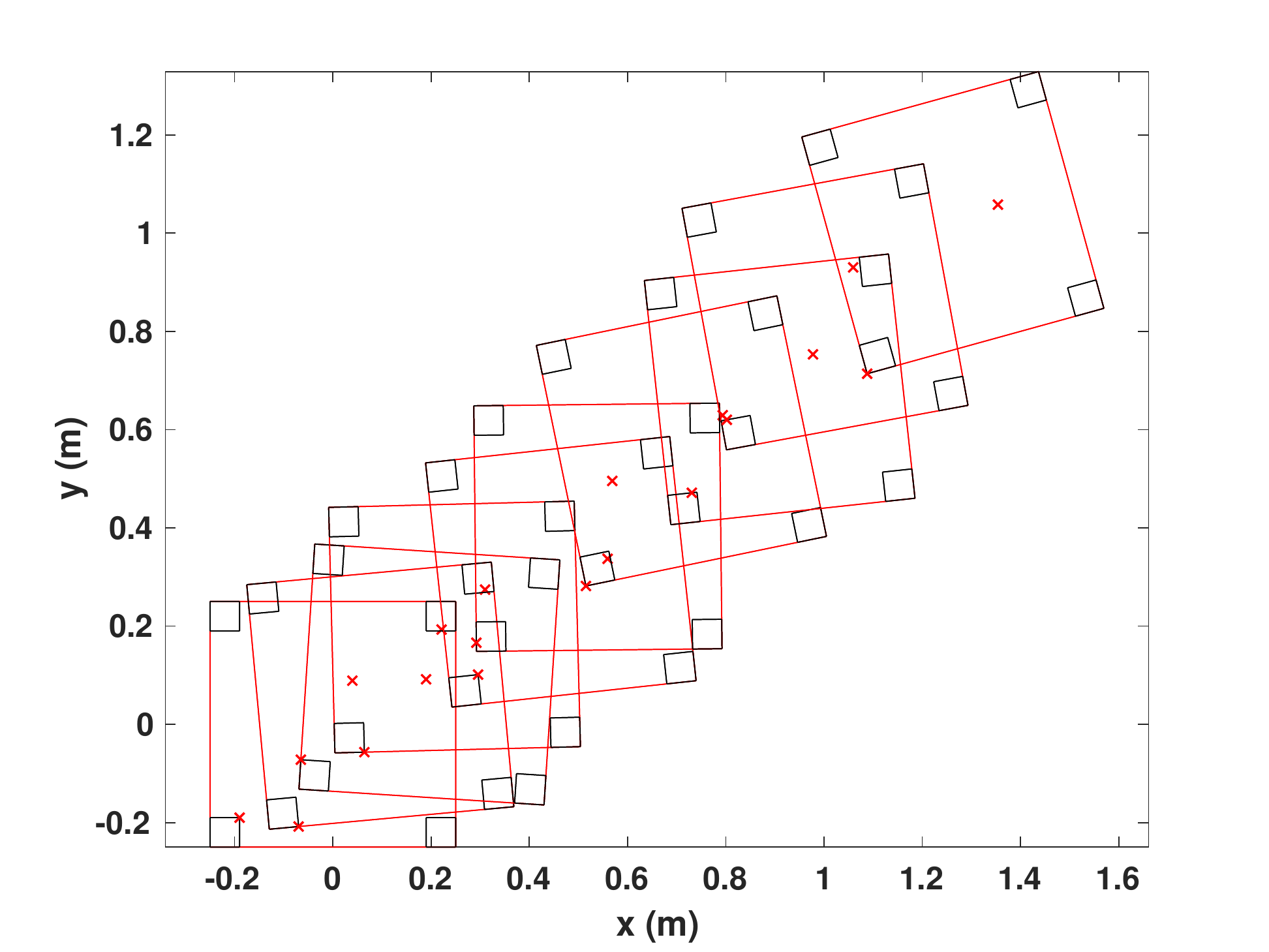}%
\caption{}
\label{figure:ex1_2} 
\end{subfigure}\hfill%
\begin{subfigure}{1\columnwidth}
\includegraphics[width=\columnwidth]{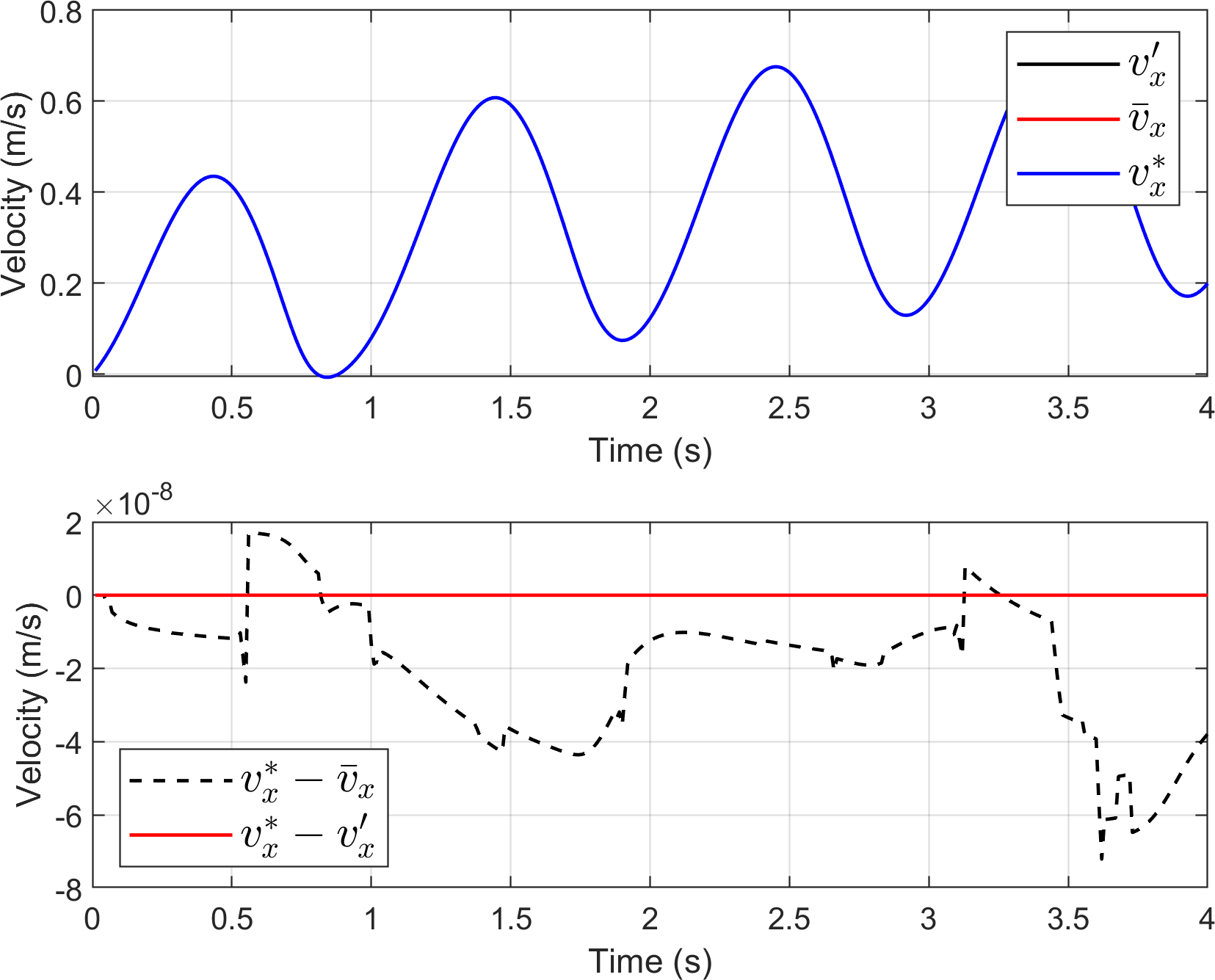}%
\caption{}
\label{figure:ex1_3} 
\end{subfigure}\hfill%
\begin{subfigure}{1\columnwidth}
\includegraphics[width=\columnwidth]{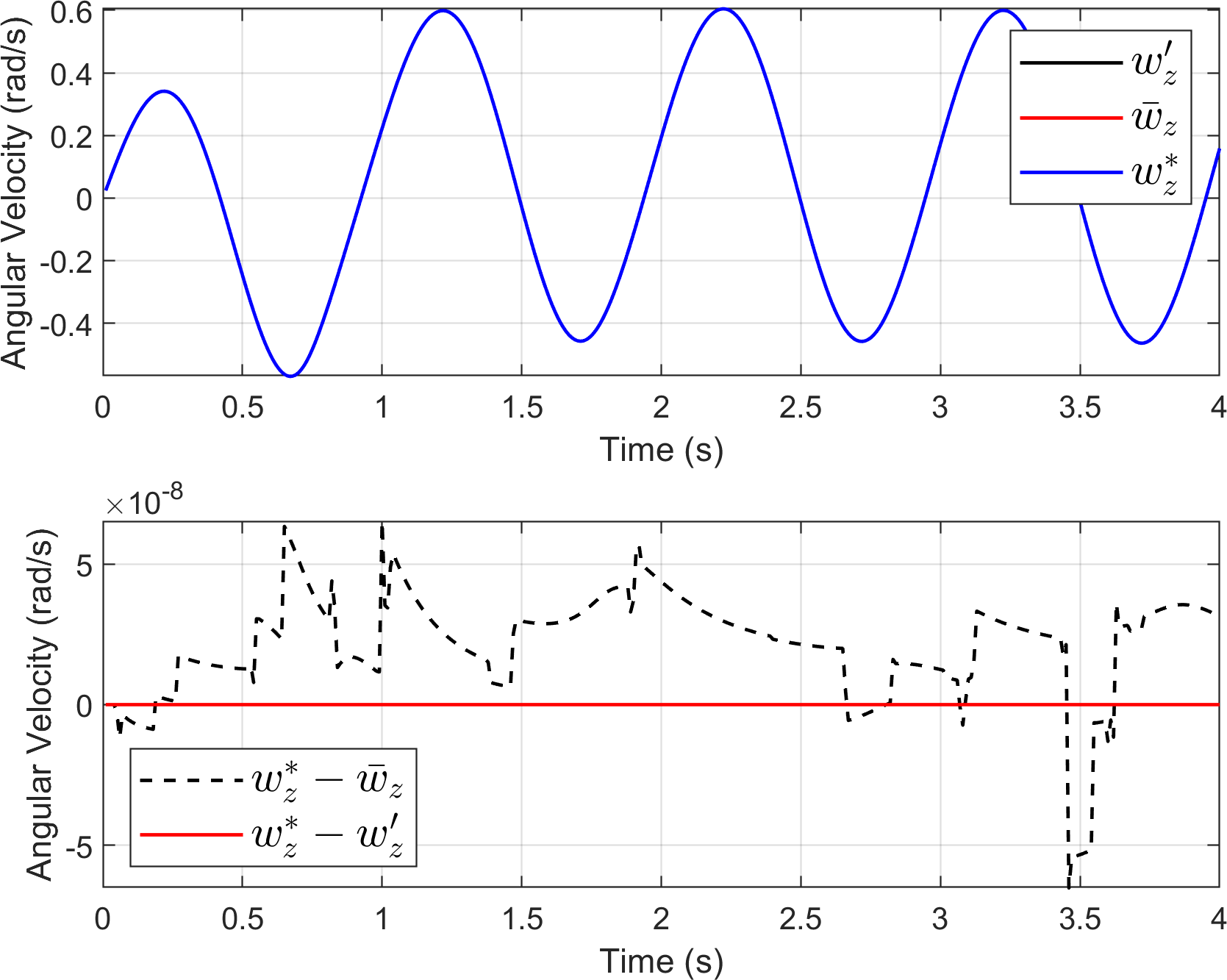}%
\caption{}
\label{figure:ex1_4} 
\end{subfigure}%
\caption{The simulation result for the scenario of pushing a desk with four legs. The figure in (a) shows a four-legged desk on a flat ground pushed by robotic grippers (left panel), and the contact between the desk feet and the ground is a union of four squares (right panel). The red square is the convex hull for the contact patch. The plot in (b) shows the snapshots of the contact patch between the feet of the desk and the ground during the motion. The red dot is the ECP, which changes during the motion, relative to the square convex hull of the contact patch. The plots in (c) (Top) illustrate the solution of translational velocity from~\cite{XieC18b} ($v^{\prime}_x$), from~\cite{XieC18a} ($\bar{v}_x$) and the proposed method ($v^*_x$), (Bottom) the difference between $v^*_x$ and $\bar{v}_x$, and difference between $v^*_x$ and $v^{\prime}_x$. In (d), we plot the (Top) The angular velocity $w^{\prime}_z$, $\bar{w}_z$ and $w^{*}_z$, (Bottom) and their differences. }
\label{Example1}
\end{figure*}
\begin{figure*}[!htp]%
\centering
\begin{subfigure}{1\columnwidth}
\includegraphics[width=\columnwidth]{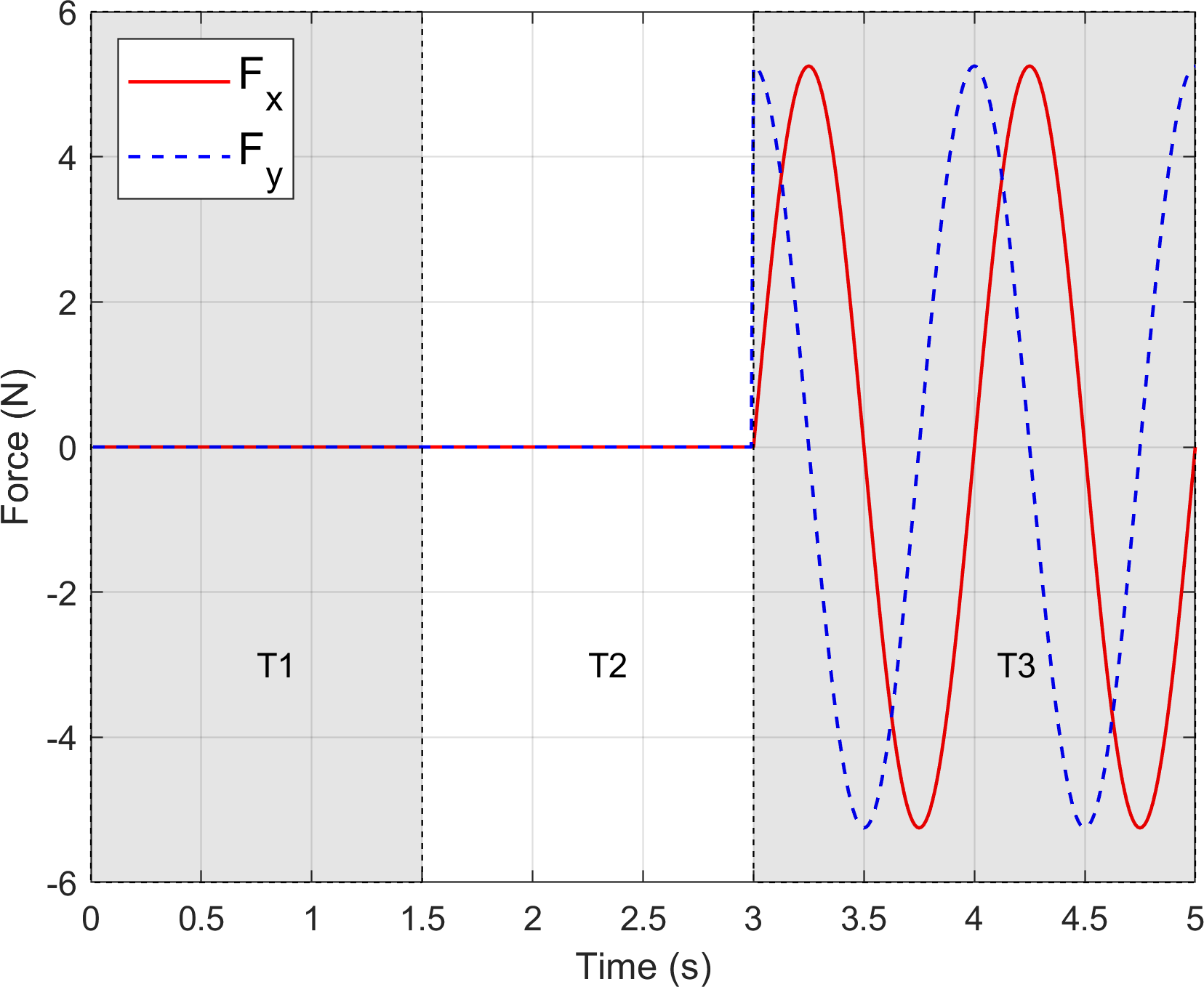}%
\caption{}
\label{figure:ex2_1} 
\end{subfigure}\hfill%
\begin{subfigure}{1\columnwidth}
\includegraphics[width=\columnwidth]{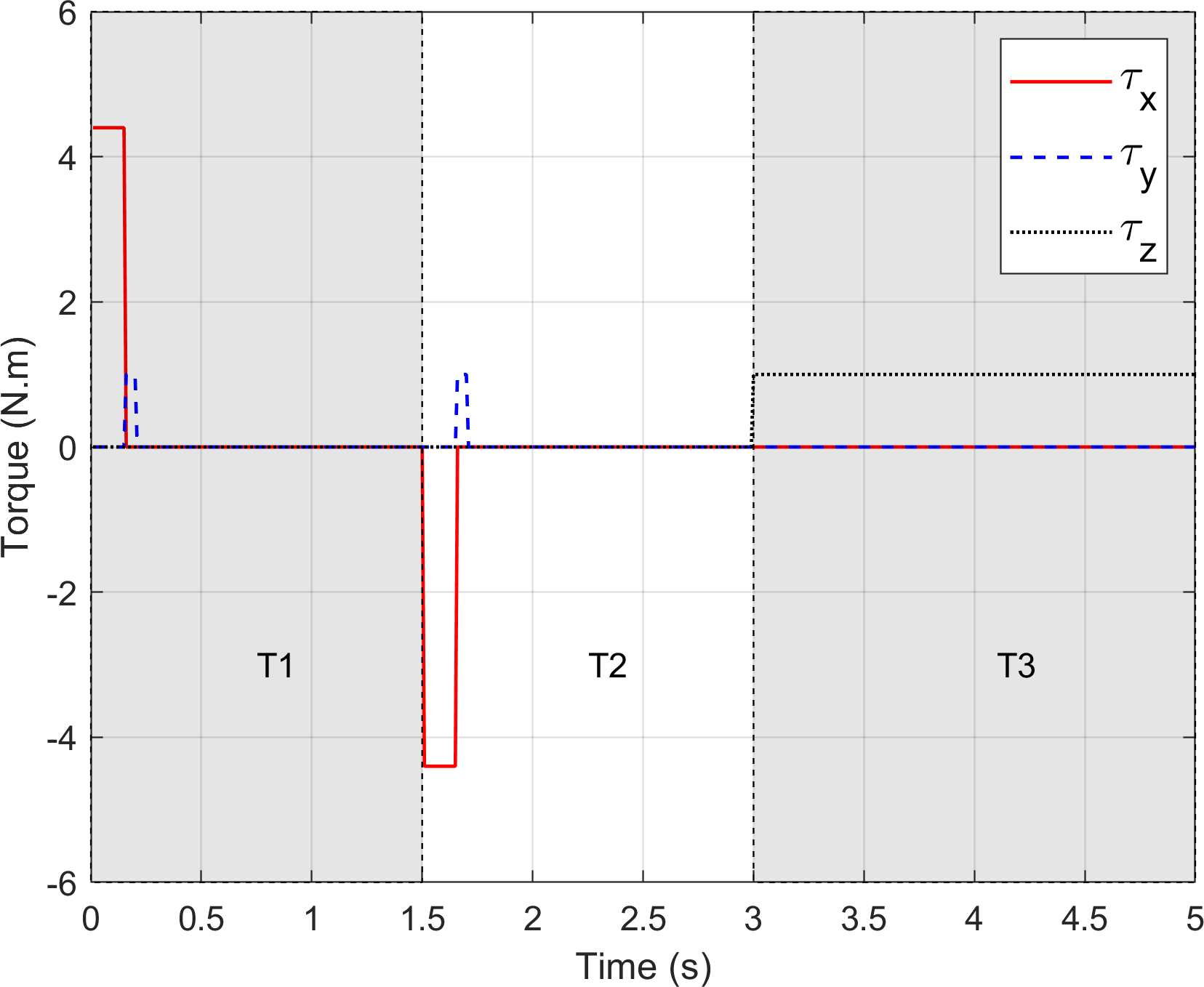}%
\caption{}
\label{figure:ex2_2} 
\end{subfigure}\hfill%
\begin{subfigure}{0.9\columnwidth}
\includegraphics[width=\columnwidth]{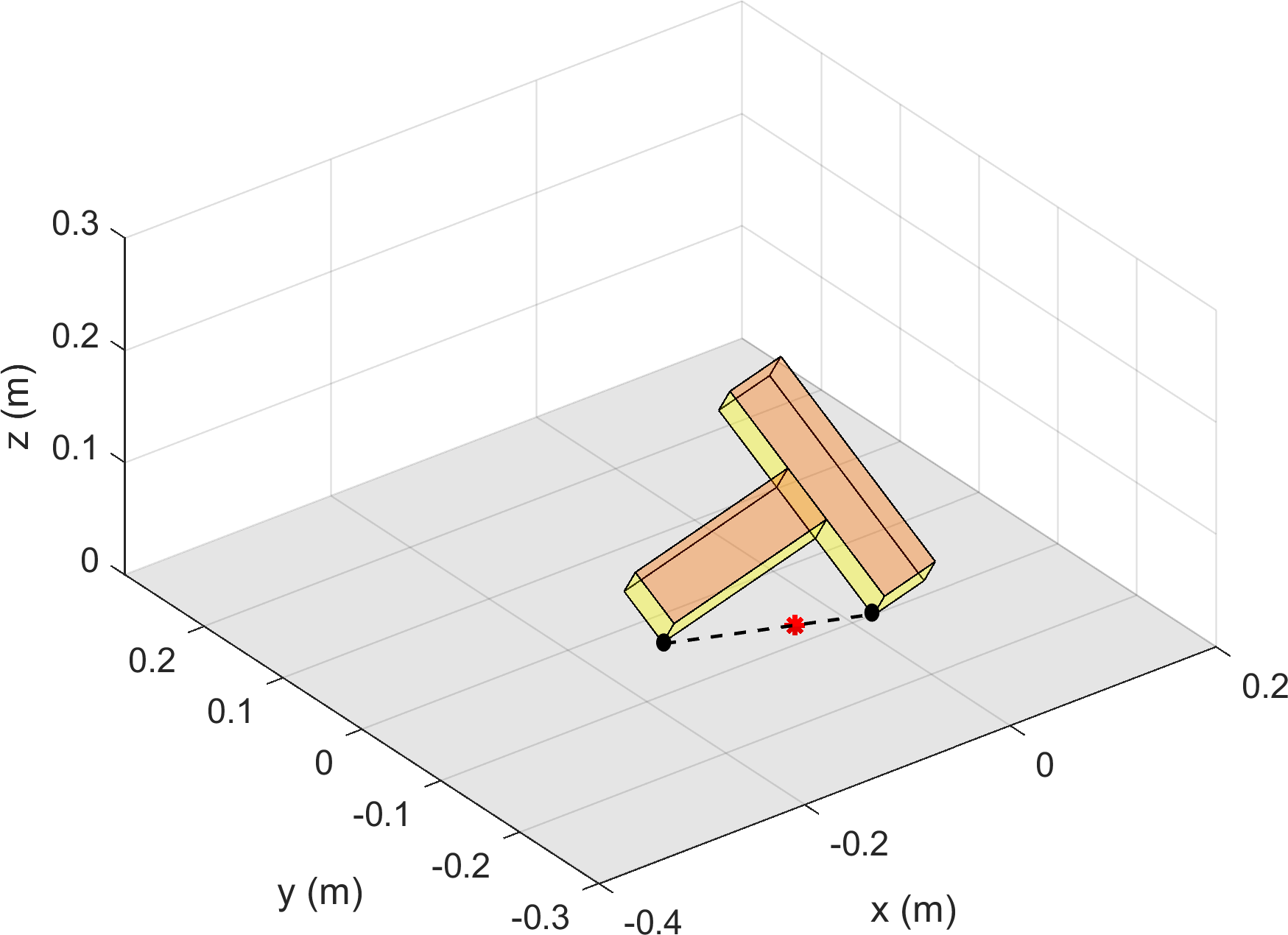}%
\caption{% (Right) The coordinates of ECP ($a_x,a_y,a_z$) during the motion.
}
\label{figure:ex2_3} 
\end{subfigure}\hfill%
\begin{subfigure}{1\columnwidth}
\includegraphics[width=\columnwidth]{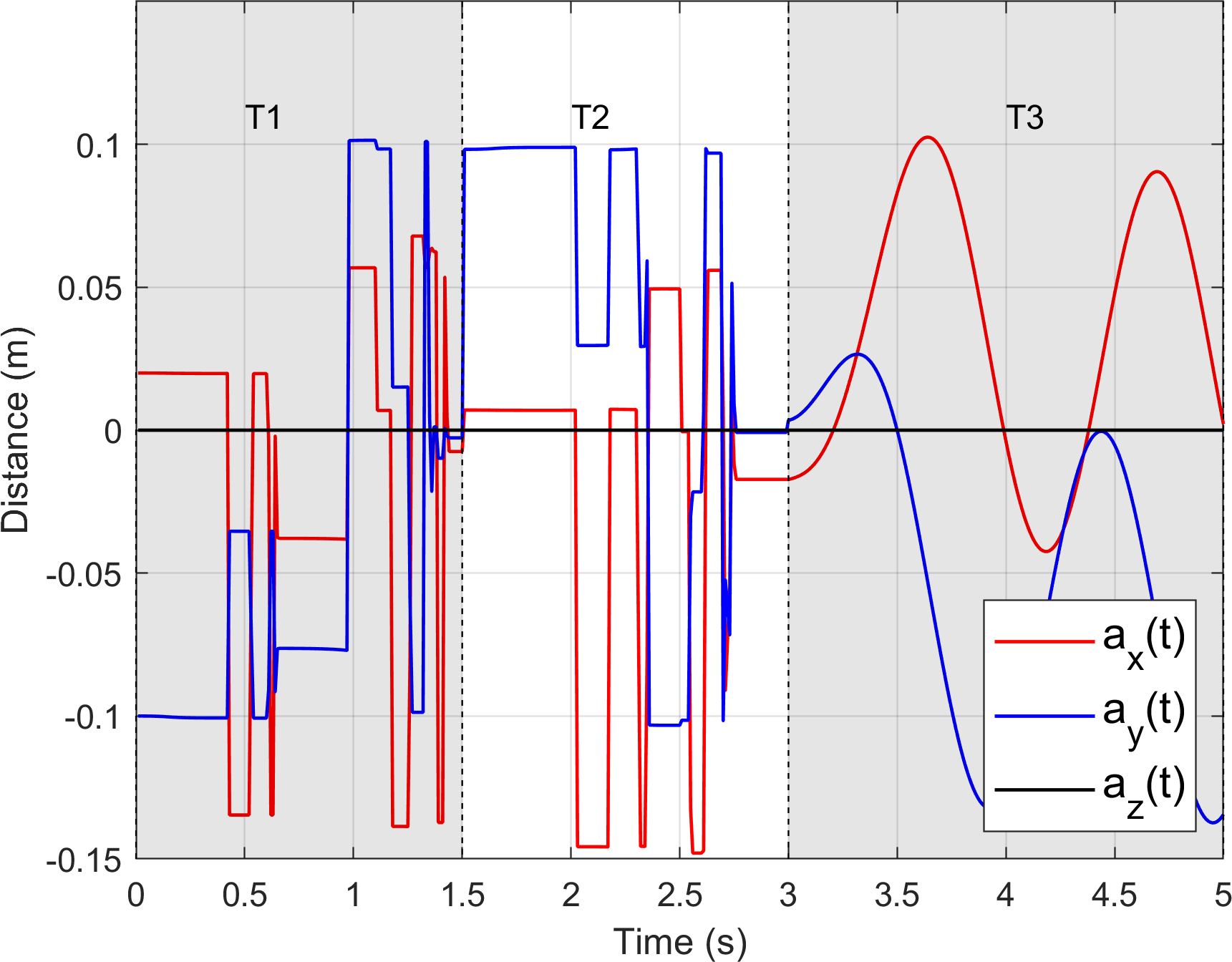}%
\caption{}
\label{figure:ex2_4} 
\end{subfigure}
\begin{subfigure}{1\columnwidth}
\includegraphics[width=\columnwidth]{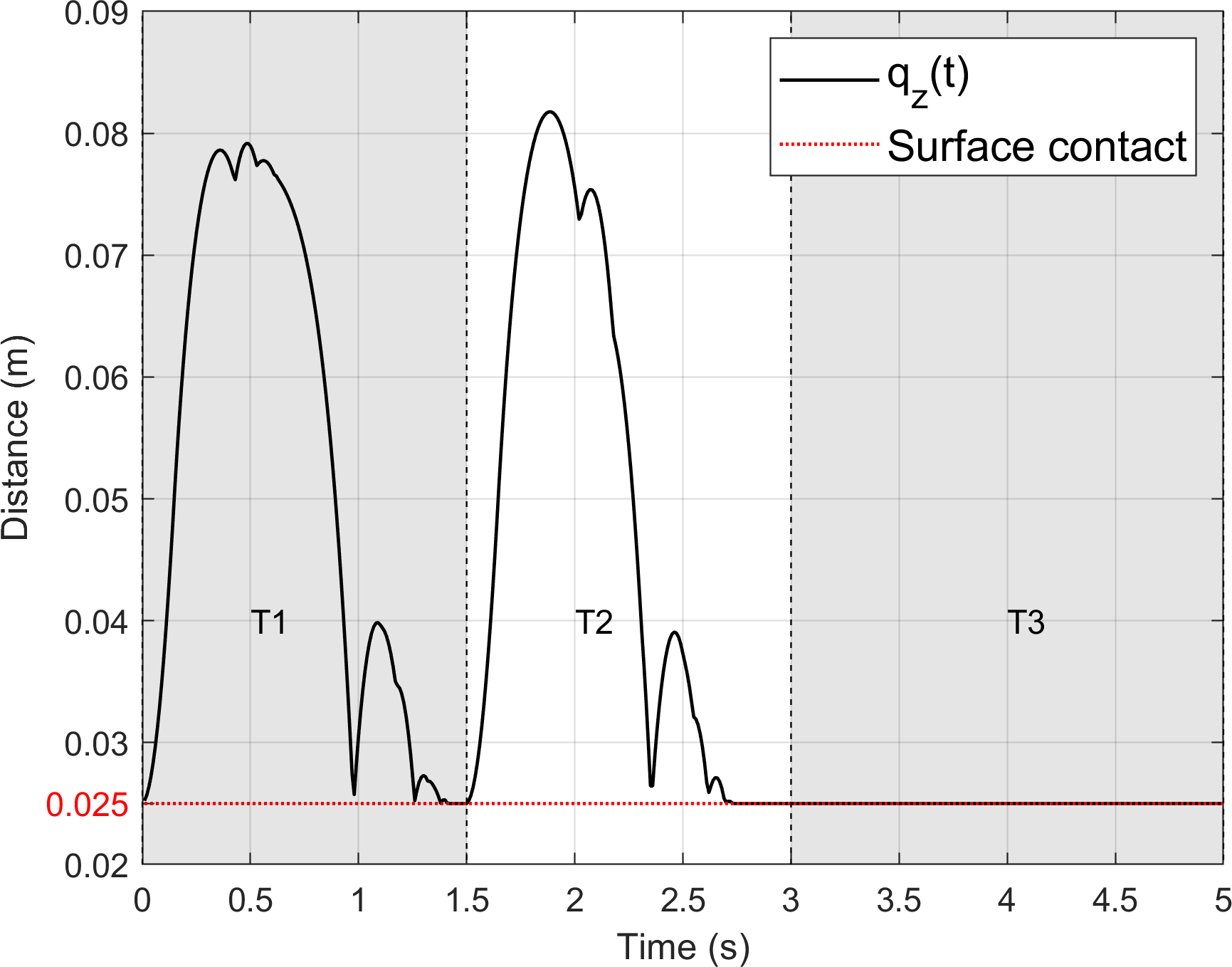}%
\caption{}
\label{figure:ex2_5} 
\end{subfigure}
\begin{subfigure}{1\columnwidth}
\includegraphics[width=\columnwidth]{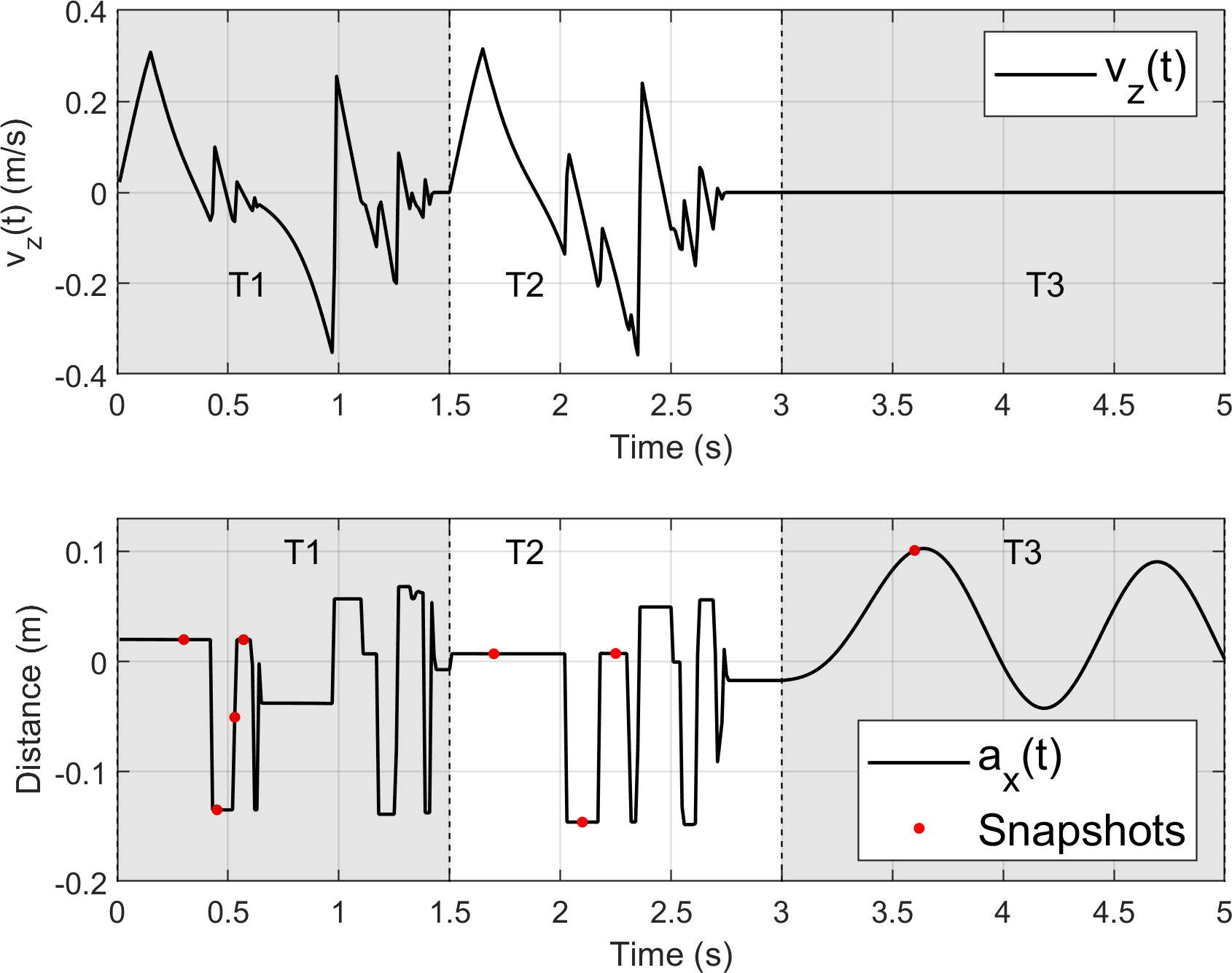}%
\caption{}
\label{figure:ex2_6} 
\end{subfigure}
\caption{ Simulation for the motion of T-shaped bar showing transitions among non-convex surface contact and non-convex two point contact. In the plots, we divide the time trajectories into three phases from T1 to T3. In (a) and (b), we plot the applied forces ($F_x,F_y$) and applied torques ($\tau_x,\tau_y,\tau_z$) and from the gripper exerted on the T-shaped bar. In (c), we plot the snapshot showing a two-point contact between the bar and ground which is non-convex. ECP (in red) lies on the line joining (i.e., convex hull of) two contact points.  (d) shows the trajectory for ECP. (e) Plot of $q_z$'s trajectory. (f) Plot of $z-$component of velocity, $v_z$, and the $x-$ coordinate of the ECP. Note that we take the snapshots at chosen timings (shown in red dots), and the plots are shown in Figure~\ref{Example2_snapshots}.}
\end{figure*}
\begin{figure*}
\begin{subfigure}{0.25\textwidth}
\includegraphics[width=\textwidth]{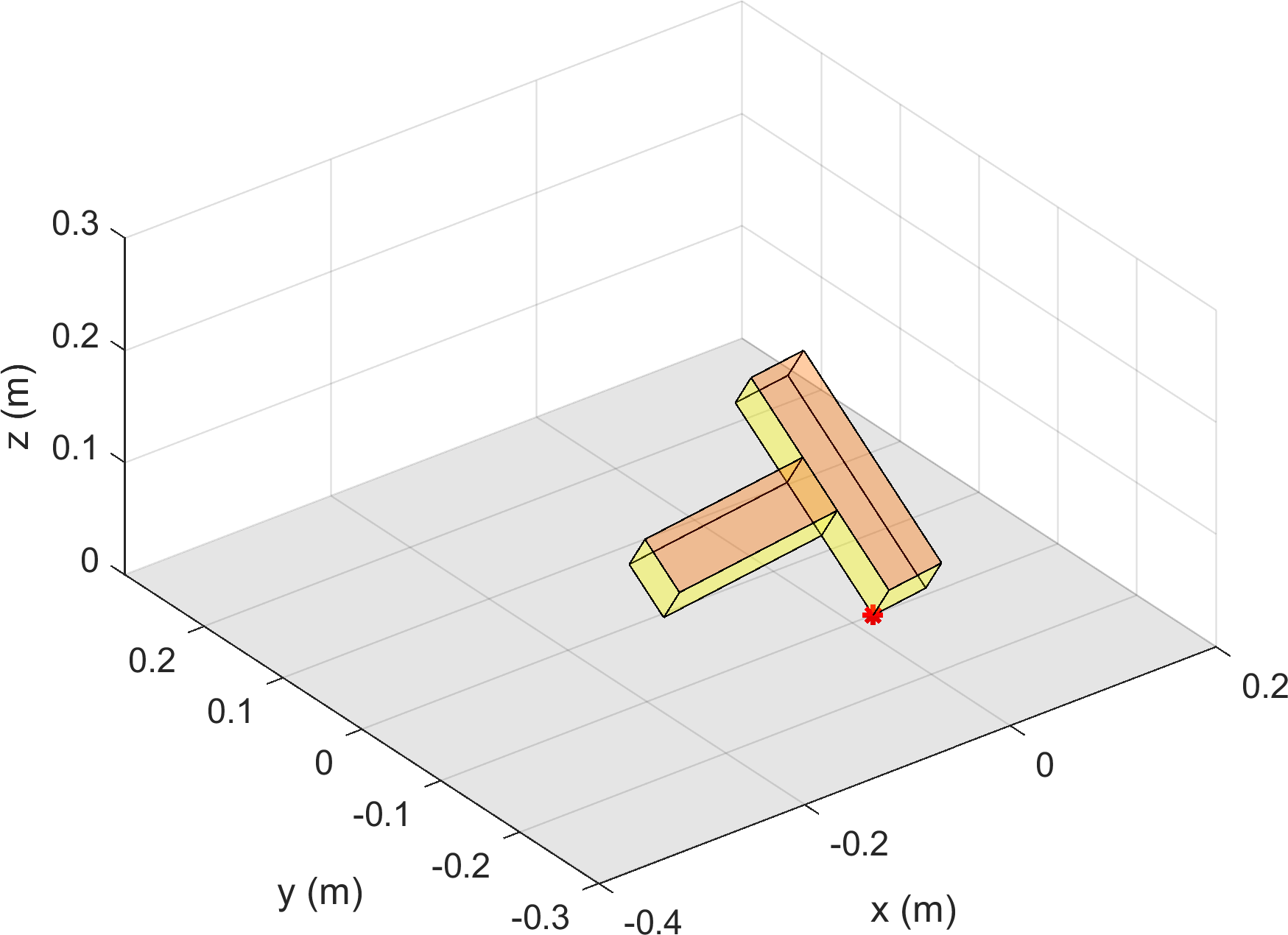}%
\caption{t = 0.30s (T1).}
\label{figure:ex2_t30} 
\end{subfigure}\hfill%
\begin{subfigure}{0.25\textwidth}
\includegraphics[width=\textwidth]{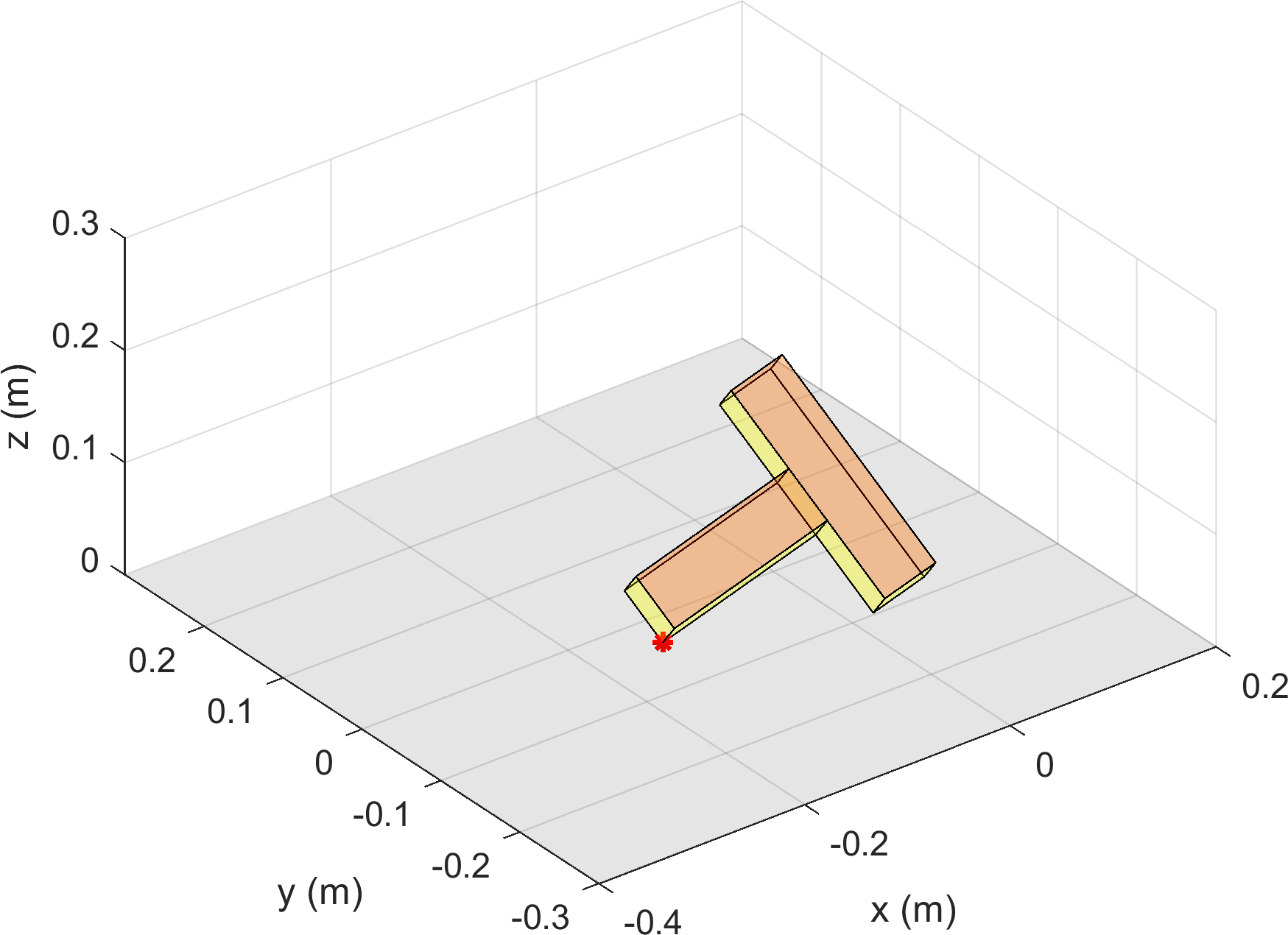}%
\caption{t = 0.45s (T1).}
\label{figure:ex2_t45} 
\end{subfigure}\hfill%
\begin{subfigure}{0.25\textwidth}
\includegraphics[width=\textwidth]{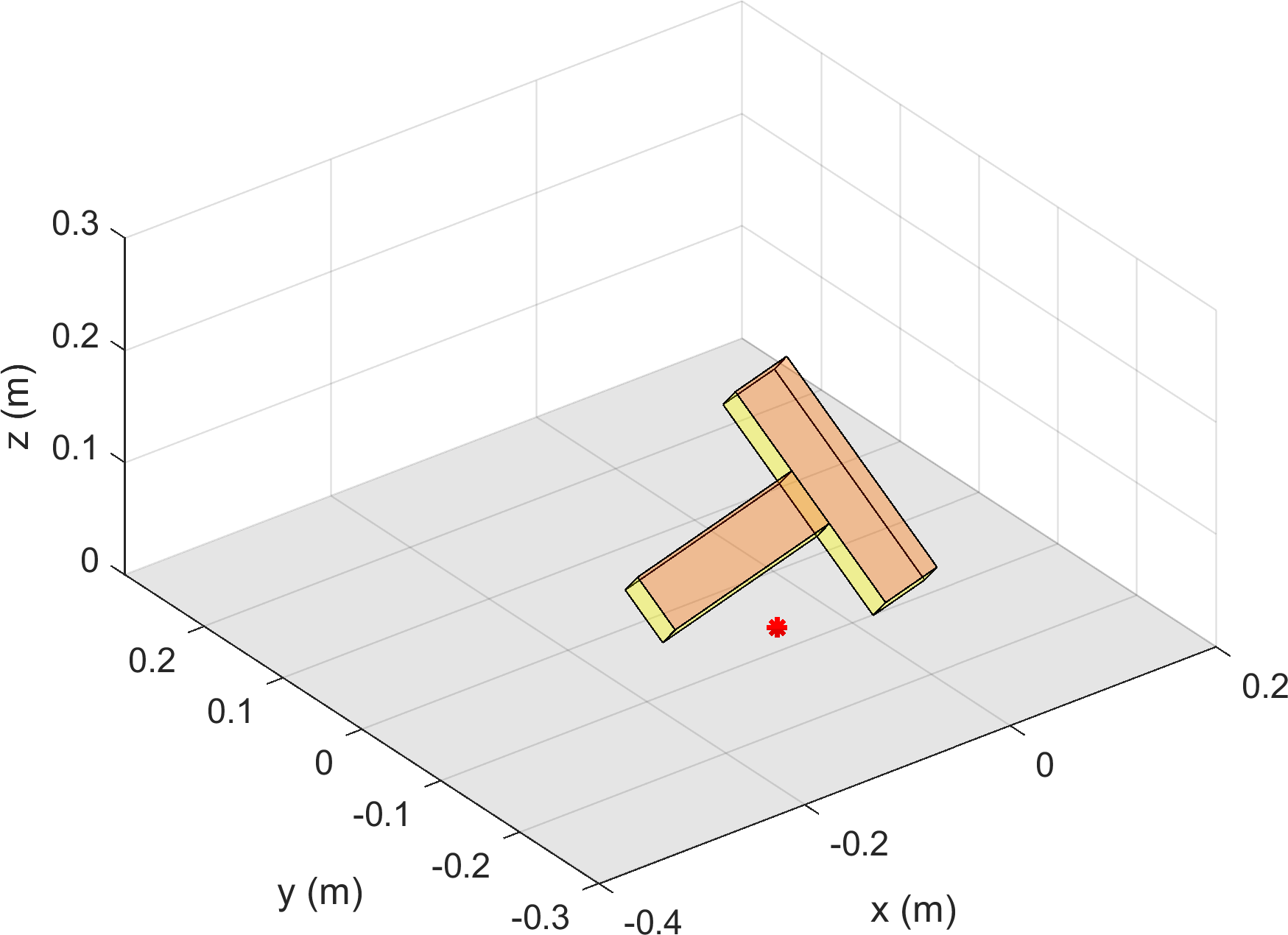}%
\caption{t = 0.53s (T1).}
\label{figure:ex2_t53} 
\end{subfigure}\hfill%
\begin{subfigure}{0.25\textwidth}
\includegraphics[width=\textwidth]{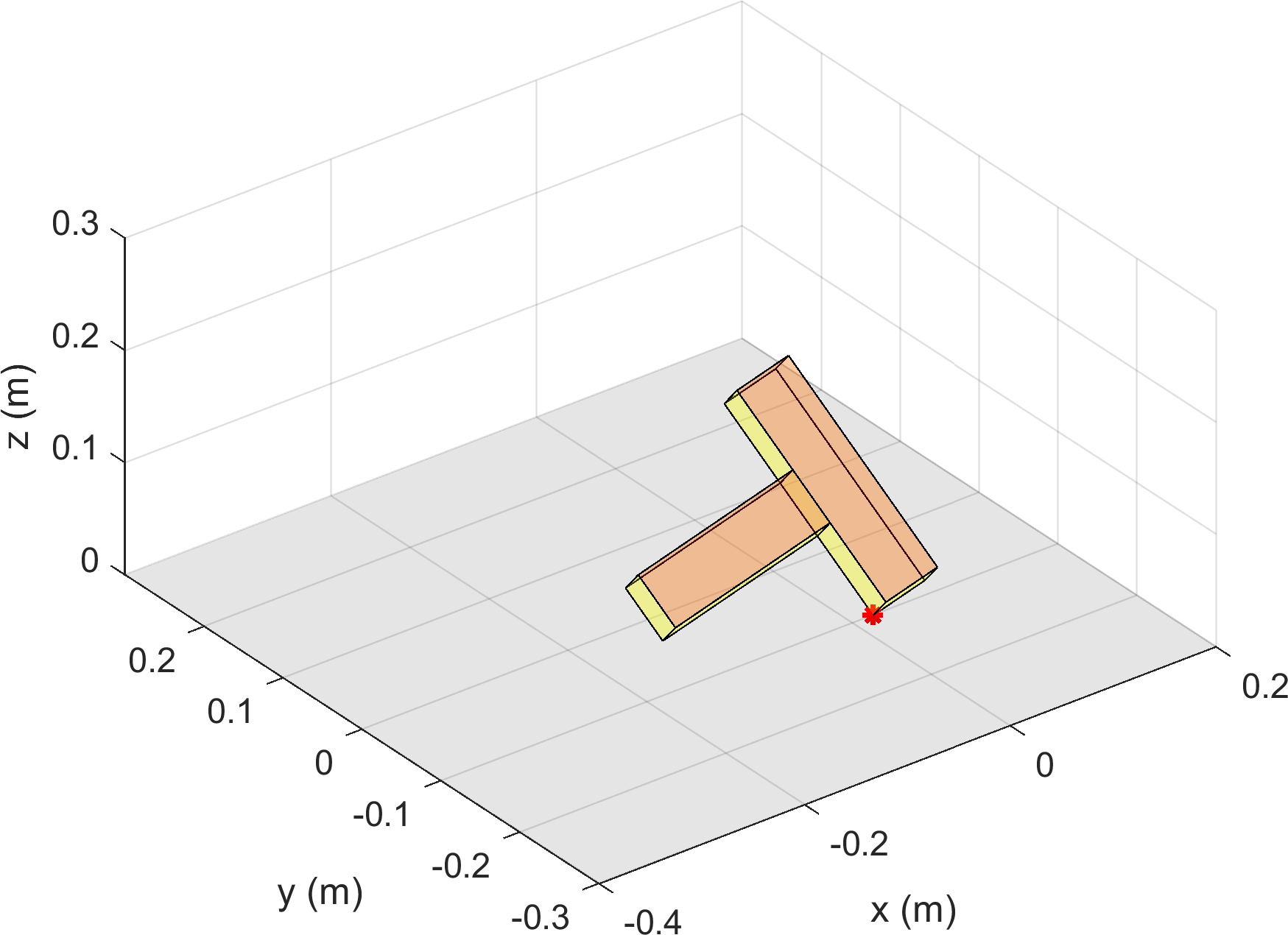}%
\caption{t = 0.57s (T1).}
\label{figure:ex2_t57} 
\end{subfigure}\hfill%
\begin{subfigure}{0.25\textwidth}
\includegraphics[width=\textwidth]{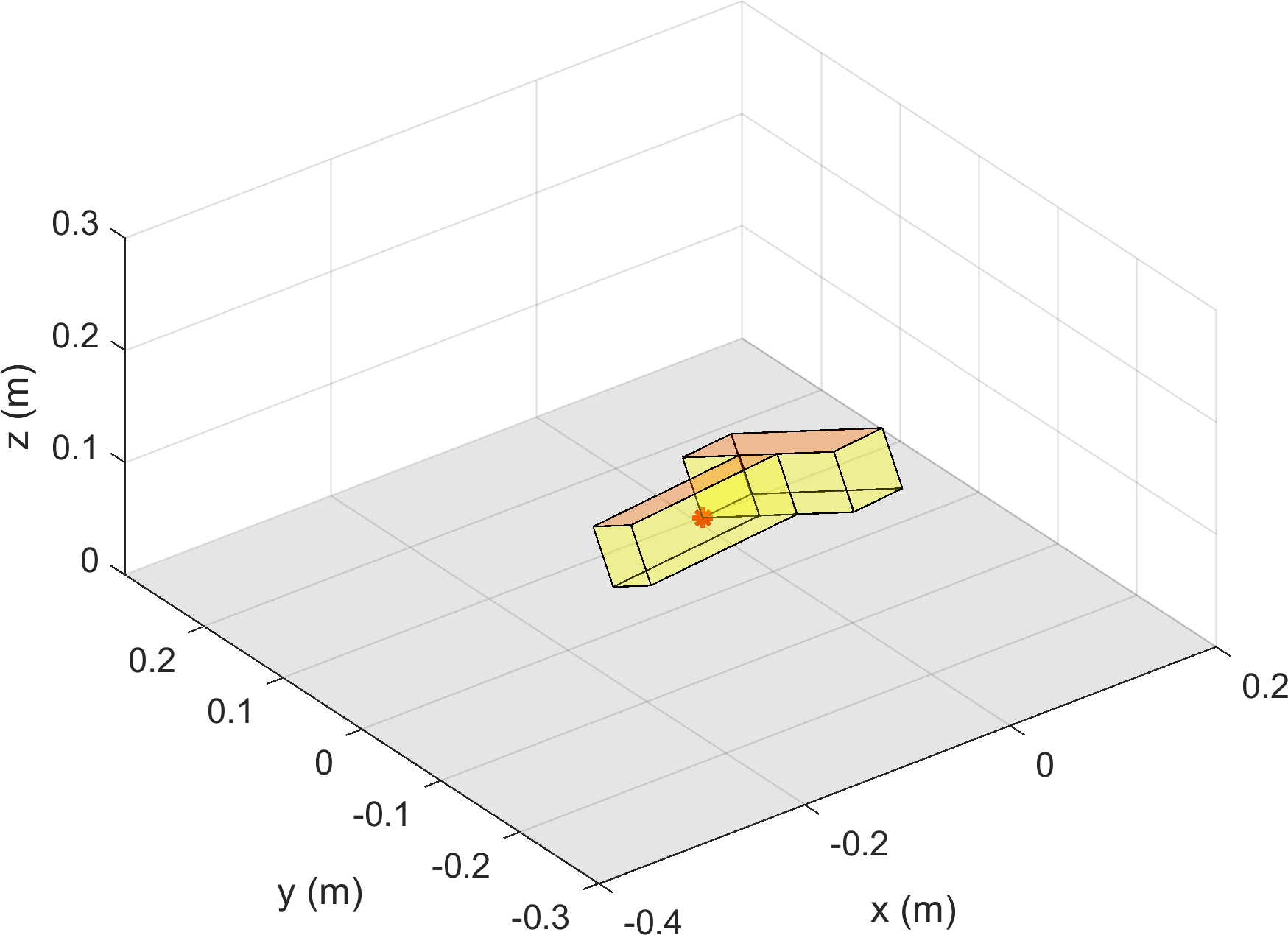}%
\caption{t = 1.7s (T2).}
\label{figure:ex2_t170} 
\end{subfigure}\hfill%
\begin{subfigure}{0.25\textwidth}
\includegraphics[width=\textwidth]{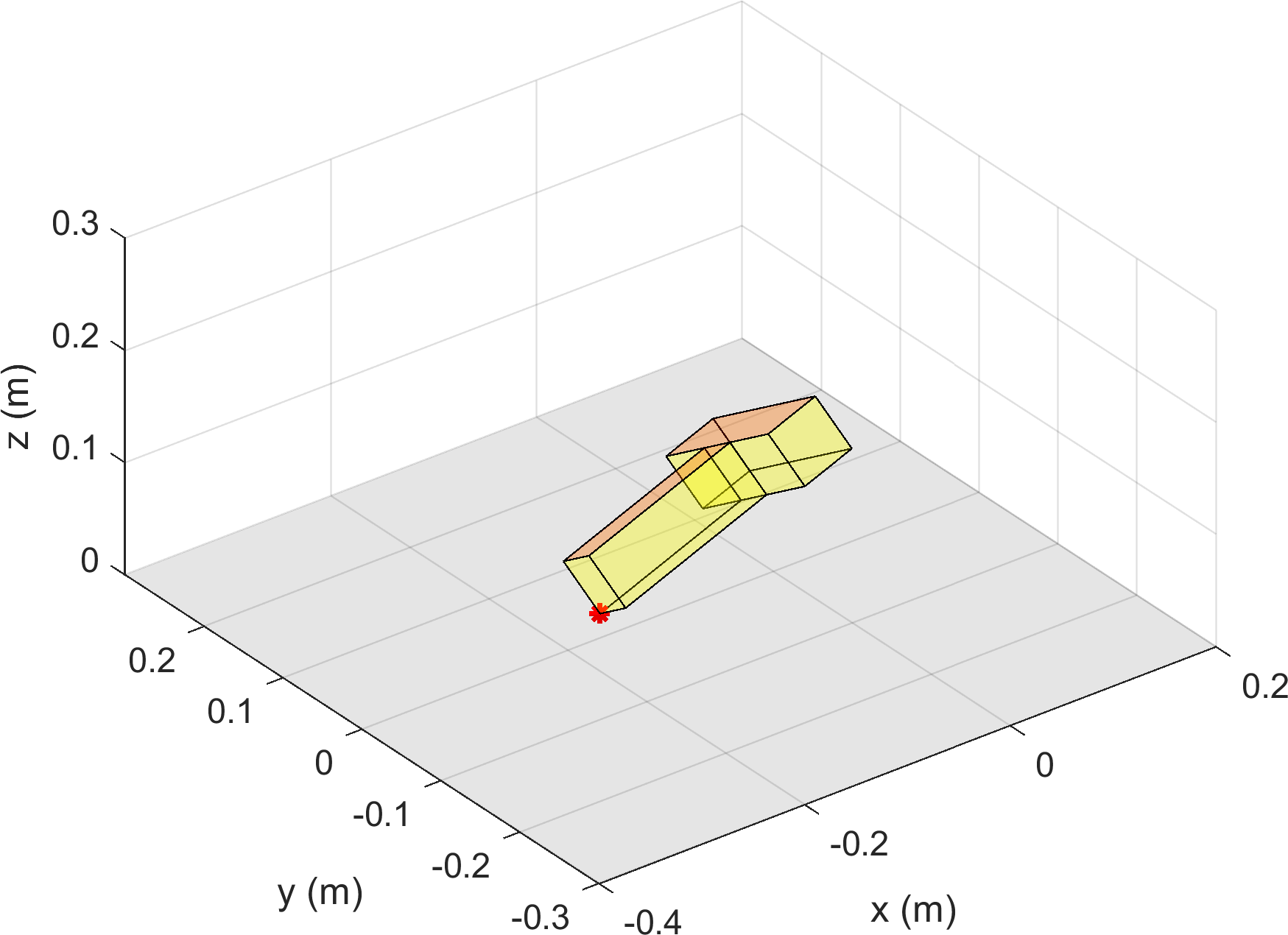}%
\caption{t = 2.10s (T2).}
\label{figure:ex2_t210} 
\end{subfigure}\hfill%
\begin{subfigure}{0.25\textwidth}
\includegraphics[width=\textwidth]{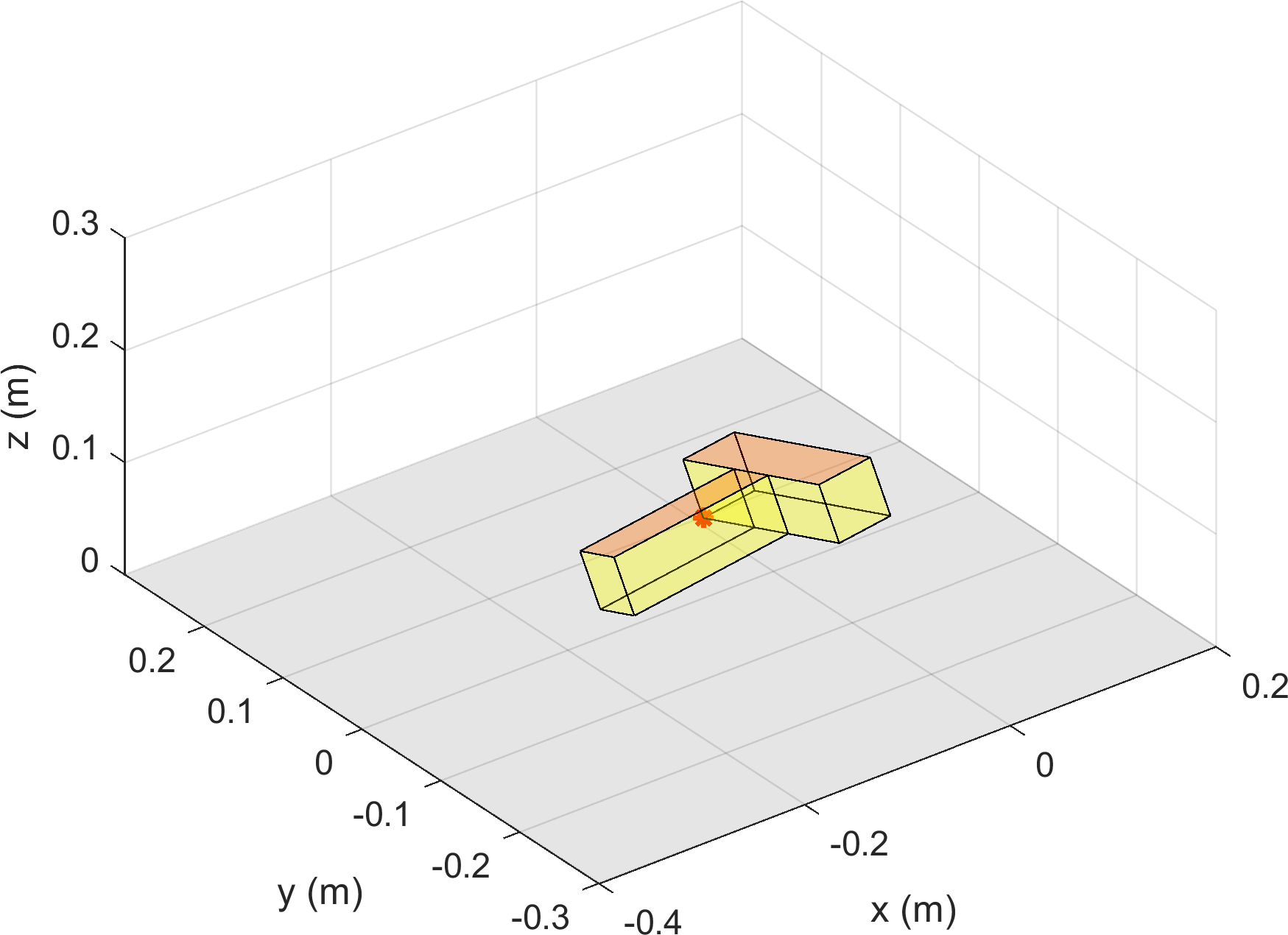}%
\caption{t = 2.25s (T2).}
\label{figure:ex2_t225} 
\end{subfigure}\hfill%
\begin{subfigure}{0.25\textwidth}
\includegraphics[width=\textwidth]{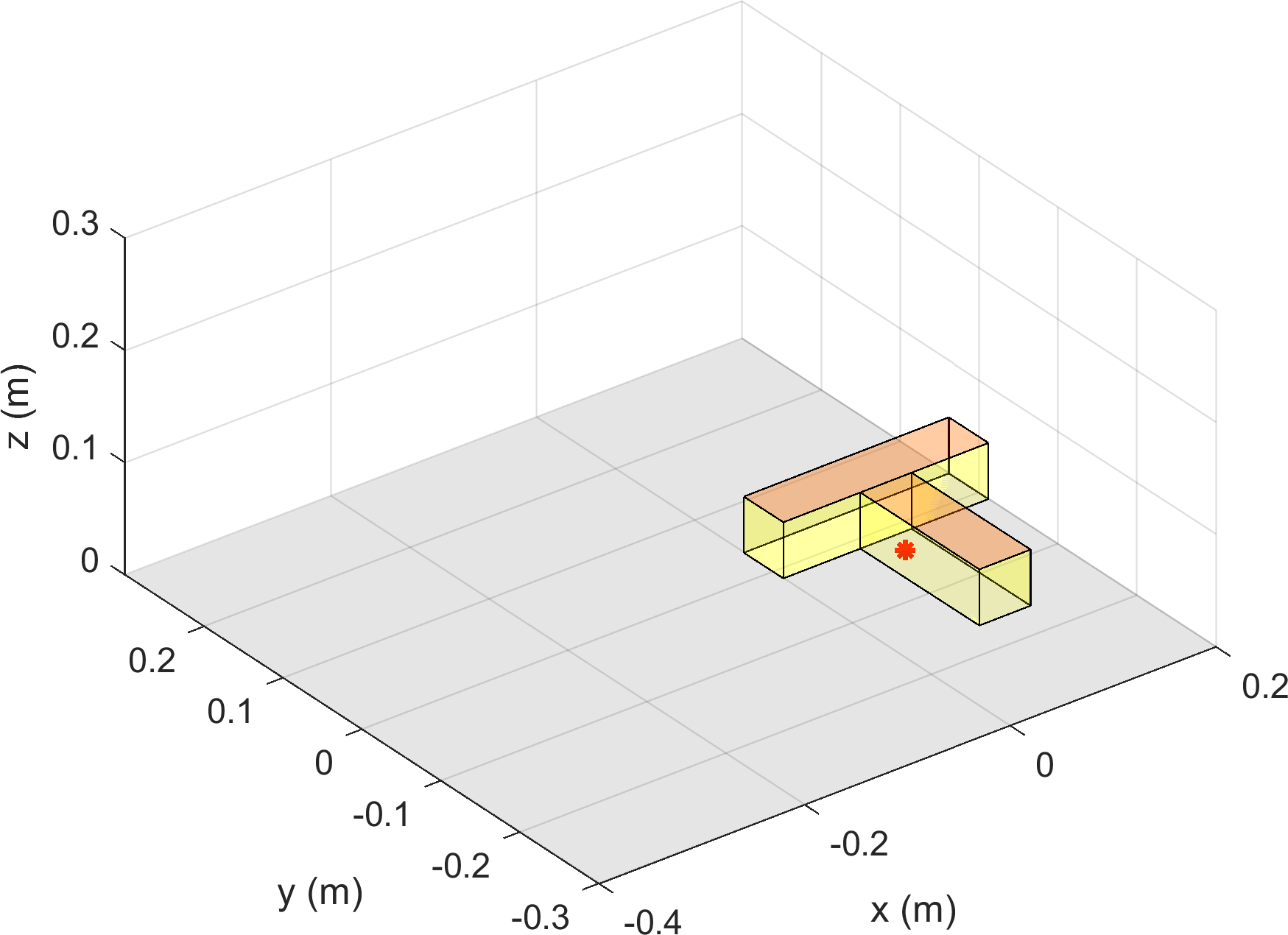}%
\caption{t = 3.60s (T3).}
\label{figure:ex2_t360} 
\end{subfigure}
\caption{ The snapshots of impact events during T1, T2 and T3 periods (The events are shown as the red dots in Figure~\ref{figure:ex2_6}). We plot the ECP (shown as the red dot) to show the transition of the bar between different contact modes. }
\label{Example2_snapshots}
\end{figure*}

\begin{figure*}[!htp]%
\begin{subfigure}{1\columnwidth}
\includegraphics[width=\columnwidth]{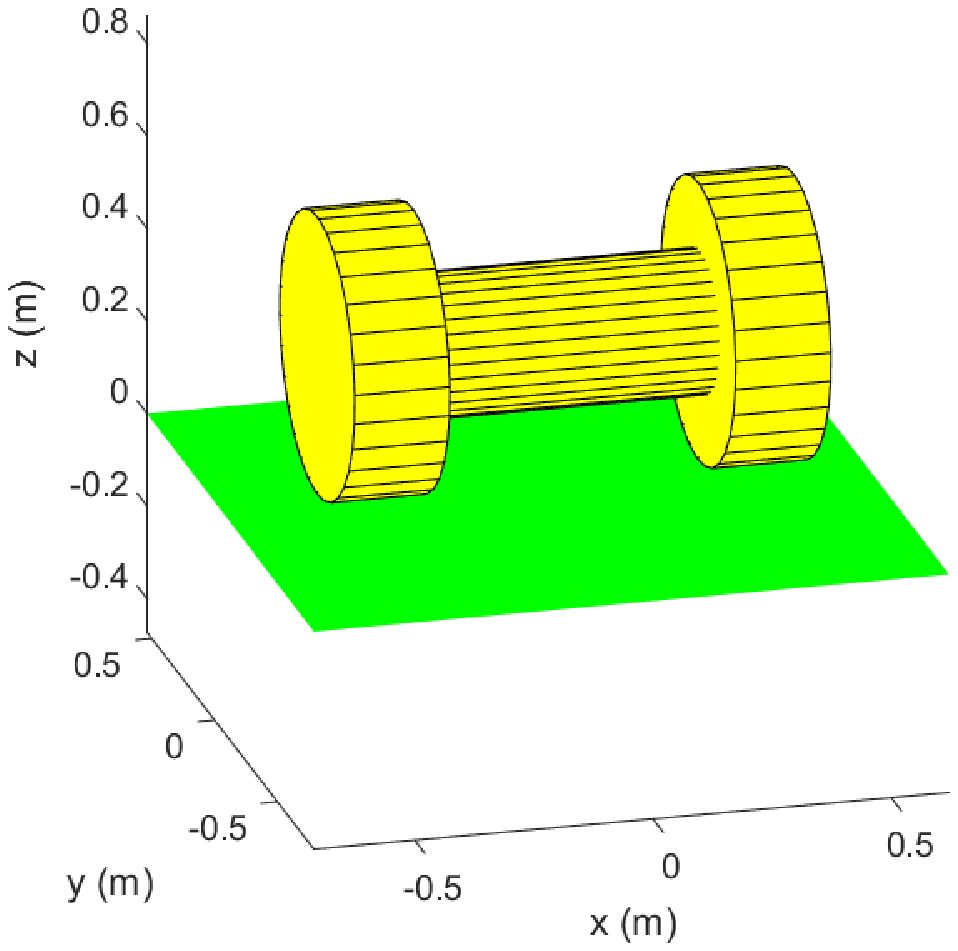}%
\caption{ }
\label{figure:ex3_1} 
\end{subfigure}\hfill%
\begin{subfigure}{1\columnwidth}
\includegraphics[width=\columnwidth]{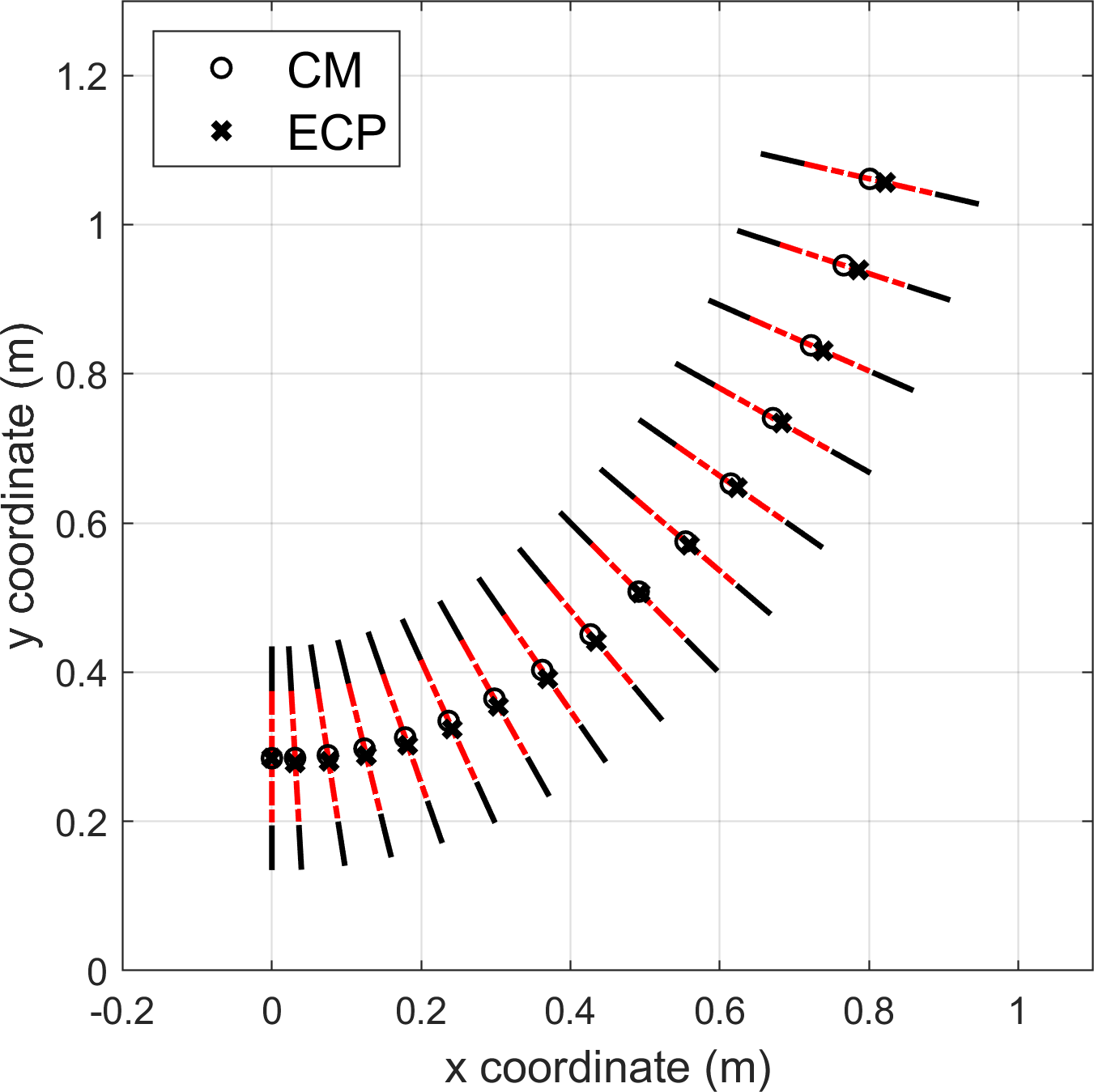}%
\caption{}
\label{figure:ex3_2} 
\end{subfigure}
\begin{subfigure}{1\columnwidth}
\includegraphics[width=\columnwidth]{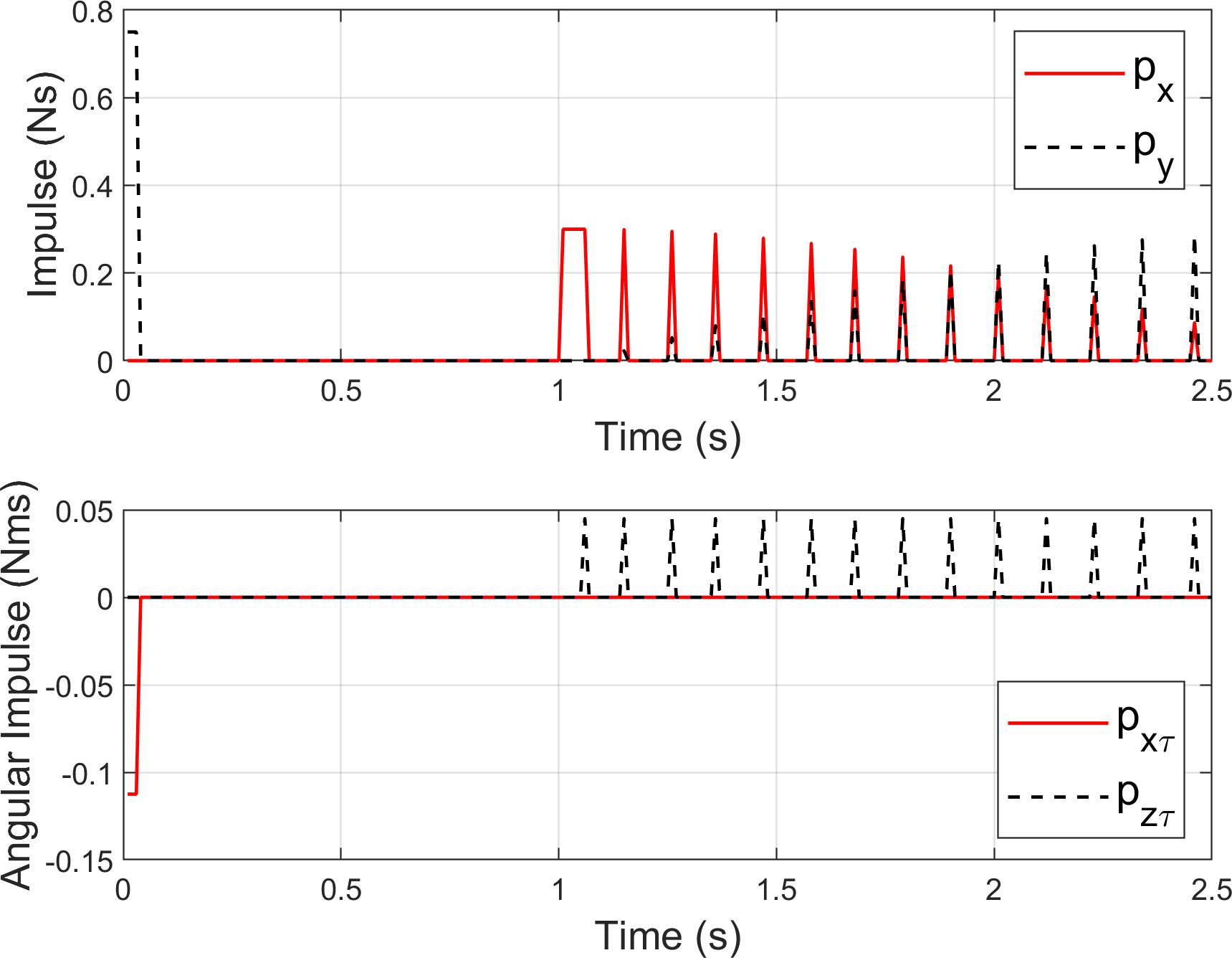}%
\caption{}
\label{figure:ex3_3} 
\end{subfigure}
\begin{subfigure}{1\columnwidth}
\includegraphics[width=\columnwidth]{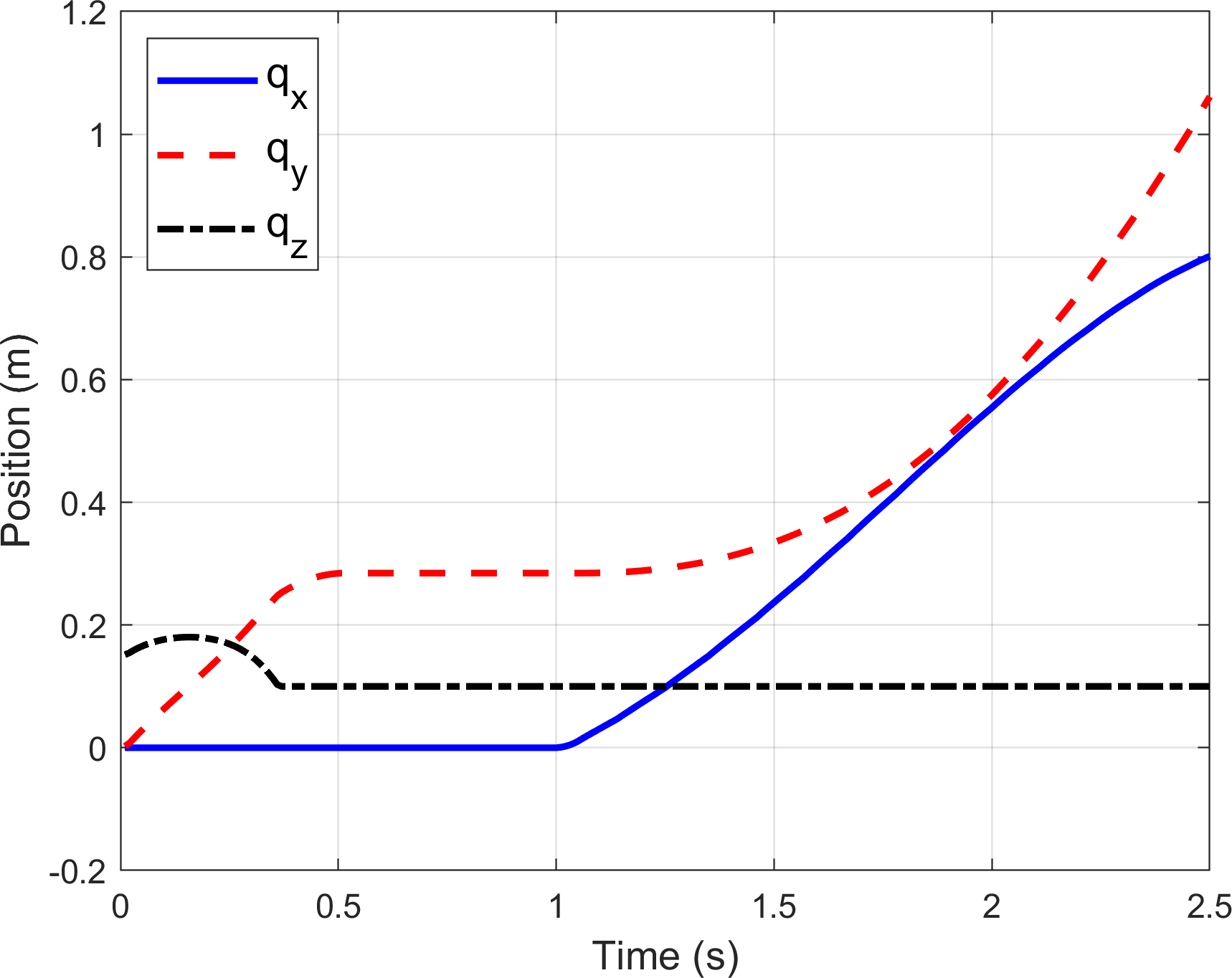}%
\caption{}
\label{figure:ex3_4} 
\end{subfigure}
\begin{subfigure}{0.2\textwidth}
\includegraphics[width=\textwidth]{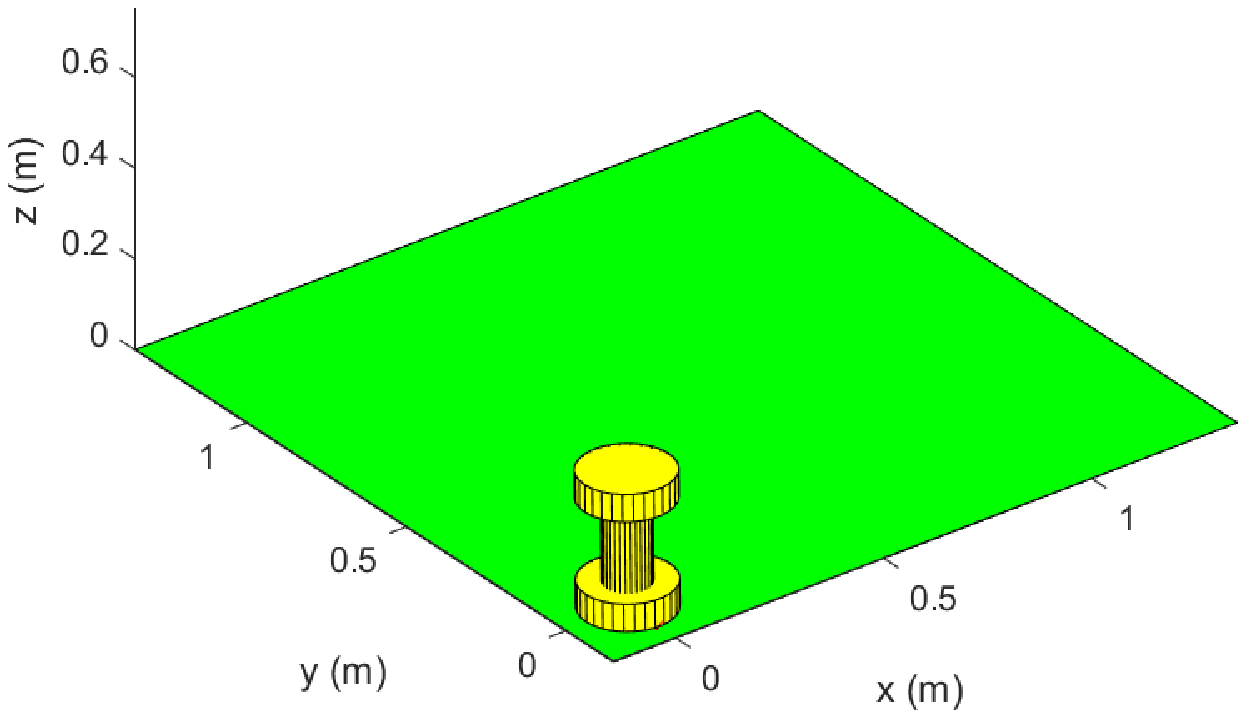}%
\caption{t = 0.01s  }
\label{figure:ex3_5} 
\end{subfigure}\hfill%
\begin{subfigure}{0.2\textwidth}
\includegraphics[width=\textwidth]{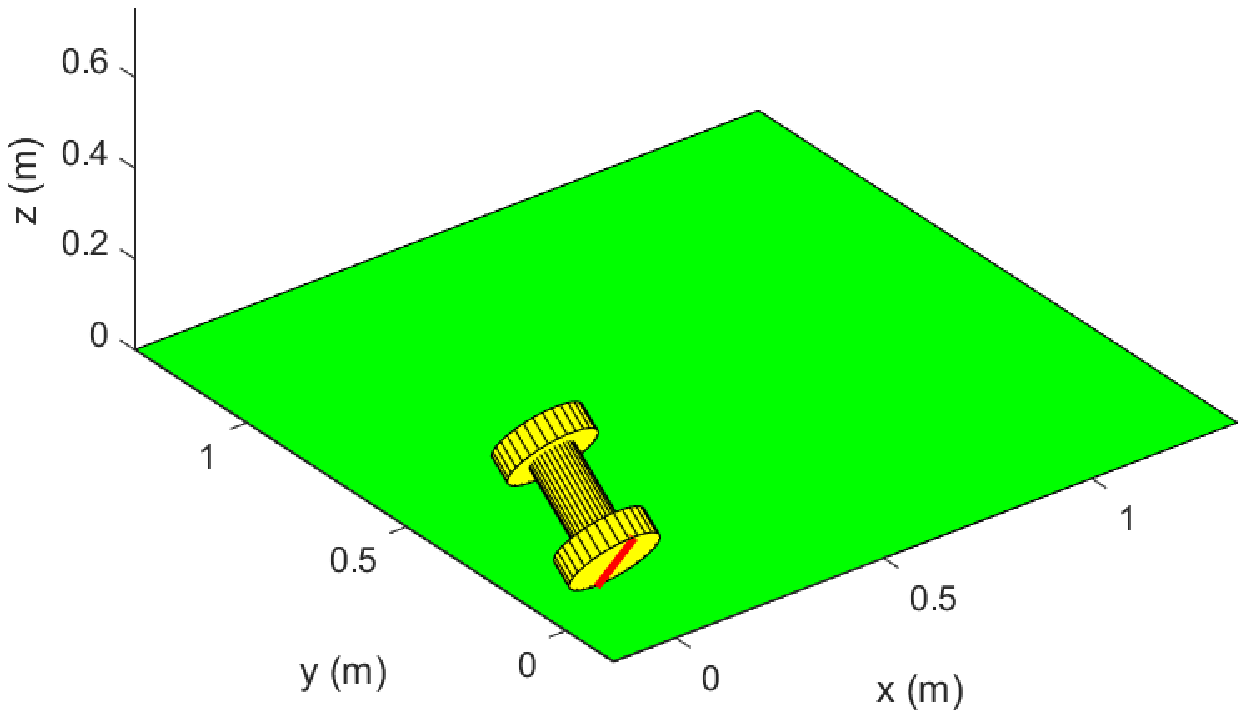}%
\caption{t = 0.25s }
\label{figure:ex3_6} 
\end{subfigure}\hfill%
\begin{subfigure}{0.2\textwidth}
\includegraphics[width=\textwidth]{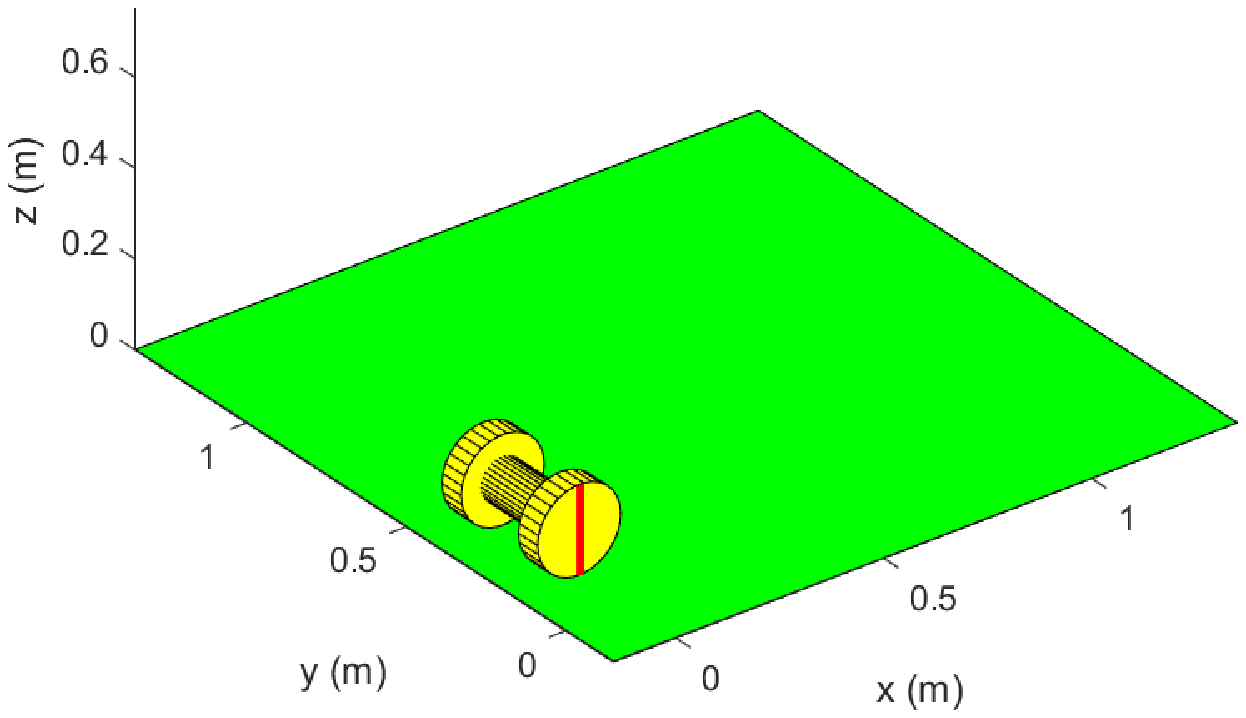}%
\caption{t = 0.50s }
\label{figure:ex3_7} 
\end{subfigure}\hfill%
\begin{subfigure}{0.2\textwidth}
\includegraphics[width=\textwidth]{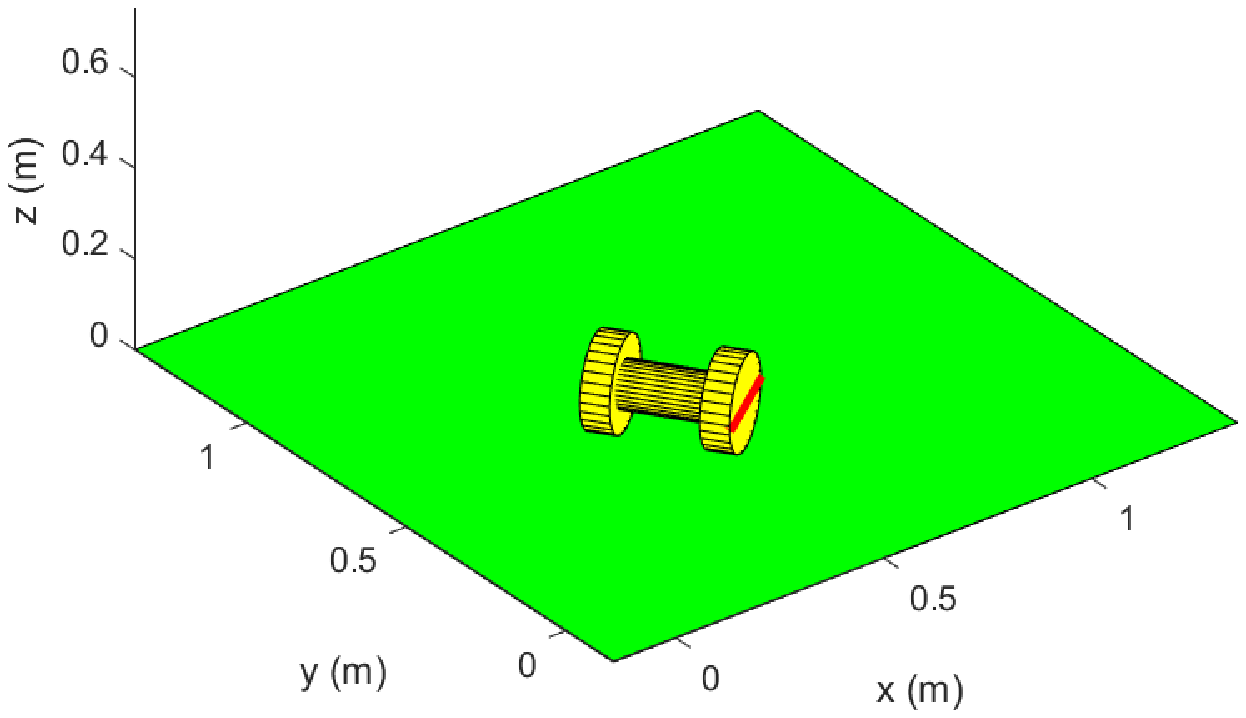}%
\caption{t = 1.75s }
\label{figure:ex3_8} 
\end{subfigure}\hfill%
\begin{subfigure}{0.2\textwidth}
\includegraphics[width=\textwidth]{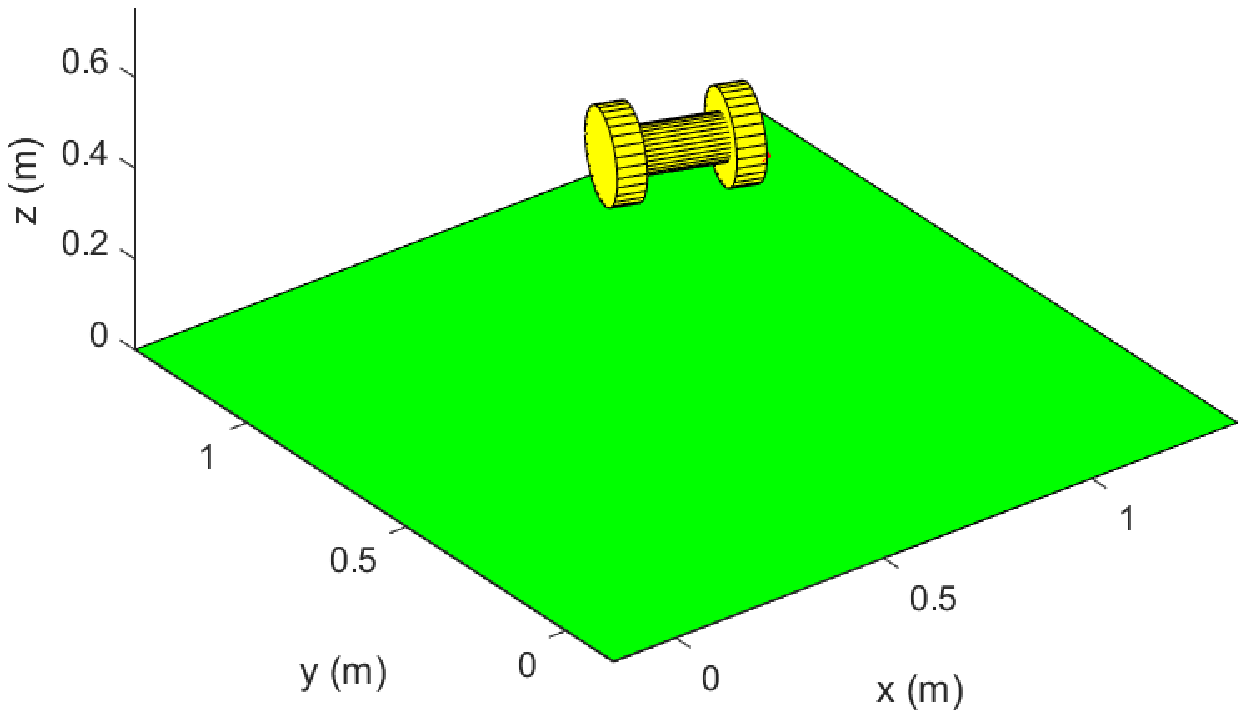}%
\caption{t = 2.5s }
\label{figure:ex3_9} 
\end{subfigure}
\caption{Simulation for the motion of dumbbell shaped object. Given the applied impulse,  the dumbbell topples and then rolls and rotates on the ground with non-convex line contact. (a) A cylindrical dumbbell shaped object where one contact patch is a union of two line segments. The convex hull for the contact patch is a line segment. (b) Snapshots of the contact patch as the dumbbell rolls and rotates on the planar support. The ECP (black cross) moves along the contact line and is usually different from the projection of center of mass. (c) Applied impulse on the dumbbell. The applied impulses $p_z$ and $p_{y\tau}$ are zero, and hence not shown.(d) The position of center of mass of the dumbbell. Note that, when the dumbbell rolls and rotates on the ground (after $t = 0.35$s), the height of CM $q_z$ stays constant as depicted by the plot for $q_z$. We plot the snapshots of dumbbell in (e) to (f).}
\label{Example3}
\end{figure*}
In this section we evaluate the performance of our proposed method on three example problems. The chosen problems are motivated by robotic manipulation scenarios, where the robot wants to manipulate the objects by exploiting contacts with the environment. The videos of the three simulation scenarios are available at~\cite{videos}. All the simulations are run in MATLAB on a MacBook Pro with 2.6 GHZ processor and 16 GB RAM.  
\subsection{Pushing a desk with four legs}
We first consider the problem of predicting motion of a square desk with four legs pushed by a robot, where the contact patch between desk feet and support is a union of four squares (see Figure~\ref{figure:ex1_1}). Such problems are useful for robots rearranging furniture in domestic environments. The dimension for the square desk is length $L = 0.5$m, the length for each small square is $L_s = 0.06$m and height of desk's CM is $H = 0.45$m. The mass of desk is $m = 15$kg and the gravity's acceleration is $g = 9.8$m/s$^2$.

The purpose of choosing this example is multi-fold. First, we want to show that we can simulate the motion of the table where the contact region is topologically disconnected and a union of four convex regions. Second, in this example, we will restrict ourselves to sliding motion of the table, i.e., we apply forces so as to ensure sliding without toppling. Thus, the set of contact points on the table do not change during motion. This is done so that we can compare the results to our previously developed method (with non-convex patch contact) for dynamic sliding motion only~\cite{XieC18b}. This is a sanity check for the predictions of the model presented in this paper since for the case of planar sliding the results of the two methods should match. We also use this example to compare with our previous effort in~\cite{XieC18a}, where we did not use the convex hull of the rigid body but considered each convex contact patch as a separate contact between the two bodies.  

%We compare the convex hull-based contact detection method presented in this paper with method in~\cite{XieC18a} and~\cite{XieC18b}. Note that we apply forces, so as to ensure sliding without toppling. Therefore, the dynamic model for sliding motion that we proposed in~\cite{XieC18b} can be applied here. In~\cite{XieC18b}, we have shown that for sliding-only motion, the discrete-time equations of motion can be reduced to a system of four quadratic equations. Since the contact is a union of four disjoint squares, we can also use the method in~\cite{XieC18a}, where, we consider each non-penetration constraint between each contact patch and the ground separately.   
%, where the (quadratic model), which verifies the accuracy and efficiency of our work. Note that there does not exists principled methods except our methods to simulate planar contact problem.

The time step chosen for all the simulations is $h = 0.01$ s and simulation time is $4$ s. The coefficient of friction between desk and support is $\mu=0.22$ and the given constants for friction ellipsoid are $e_t = e_o = 1$, $e_r = 0.1$ m. As shown in Figure~\ref{figure:ex1_1}, the desk slides on the support. The initial position of CM is $q_x = q_y = 0$ m, $q_z = 0.45 $m and orientation about normal axis is $\theta_z = 0$ degree. The initial velocity is $v_x = 0.3$ m/s, $v_y = 0.2$ m/s, $w_z = 0.5$ rad/s. The external forces and moments from grippers exerted on the desk is periodic, $f_x = 22.5\sin(2\pi t)+ 22.5$ N, $f_y = 22.5\cos(2\pi t) + 22.5$ N, $\tau_z = 2.1\cos(2\pi t)$ Nm, where $t \in [0,4]$ s.

In Figure~\ref{figure:ex1_2}, we show the snapshots for the contact patch during the motion. It can be seen that the table translates as well as rotates during motion. The ECP is marked by a red cross and it can be seen that the ECP is not within the contact patch and it is also not below the center of mass of the table (which matches the intuition, since the table is rotating).  In the first row of Figure~\ref{figure:ex1_3}, we plot the $x$-component of the linear velocity, $v_x$. In the first row of Figure~\ref{figure:ex1_4}, we plot the angular velocity about the normal to the plane, $w_z$. In the plots, we define $v^{\prime}_x$, $w^{\prime}_z$ as the solutions of the method in~\cite{XieC18b}, $\bar{v}_x$, $\bar{w}_z$ as the solutions of the method in~\cite{XieC18a}, and $v^*_x$, $w^*_z$ as the solutions of our proposed method in this paper. We observe that the solutions of the method in~\cite{XieC18b, XieC18a} and this paper coincide each other. The difference in $v_x$ and $w_z$ between the proposed method and the methods in~\cite{XieC18b, XieC18a} are shown in the second row in Figures~\ref{figure:ex1_3} and~\ref{figure:ex1_4}. We can observe that the differences for $v_x$ and $w_z$ are within $1e-8$, which validates the accuracy of the proposed method numerically.

%We observe that the solutions of the method in~\cite{XieC18b, XieC18a} and this paper coincide with each other. Therefore, we plot the difference between the solutions. As the second row in Figures~\ref{figure:ex1_3} and~\ref{figure:ex1_4}  illustrates, the differences for $v_x$ and $w_z$ are within $1e-8$, which validates numerically the accuracy of the proposed method. }\comment{Note that, both convex hull method and quadratic model utilize the friction model with single contact patch, while the MNCP model utilizes the friction model with multiple patches.}

Furthermore, the average time the model in~\cite{XieC18b} spends for each time step is $0.0022$s. The time our proposed method method spends is $0.0053$s (which is $2.4$ times than~\cite{XieC18b}), and the time the model in~\cite{XieC18a} spends is $0.0487$s (which is more than $22$ times than quadratic model's and $9$ times than the current method). To summarize, the proposed method simplify the model in~\cite{XieC18a} greatly by modeling multiple contact patches with a single patch and therefore is much more efficient without sacrificing accuracy. The model in~\cite{XieC18b}, although faster is valid only for sliding and cannot be applied to situations where the object may topple.% as shown in the next example.
%assumes sliding motion only and ignore the geometry of contact patch. Therefore, it is the most efficient method among three methods. 

\subsection{Manipulating a T-shaped bar}
This example is used to illustrate that our method allows objects to automatically transition between different contact modes (surface, point, line and also making and breaking of contact), while ensuring the objects do not penetrate. As Figure~\ref{figure_Motivation} illustrates, the planar contact patch between the T-shaped bar and the support is non-convex. The dimensions of the bar are given in Figure~\ref{figure_Motivation}. The mass of the bar is $2$ kg, the other parameters like gravity and friction parameters are the same as in the first example. The time step chosen is $h = 0.01$ s and the total simulation time is $t = 5$ s.

In this scenario, we first make the T-bar tilt and wobble twice on the ground, which can be divided into the phase T1 (from $t = 0.01$s to $t = 1.5$s) and the phase T2 (from $t = 1.51$s to $t = 3.0$s). Then, during the phase T3 (from $t = 3.0$s to $t = 5.0$s), we make the T-bar slide and rotate with surface contact on the plane. Figures~\ref{figure:ex2_1} and~\ref{figure:ex2_2} show the applied forces  and moments  from the gripper acting on the bar. Figure~\ref{figure:ex2_4} shows the variation of the coordinates of ECP (i.e., $a_x,a_y,a_z$) with time. Note that the coordinate of ECP along $z$ axis, i.e., $a_z$, stays zero within the numerical tolerance of $1e^{-12}$ during the motion. Thus, there is no penetration between the bar and ground. Besides, this implies that the contact between the T-bar and the ground is always maintained during the motion. Thus, the T-bar does not bounce on the ground (as should be the case, given the implicit assumption of plastic collision during impact). Furthermore, the jumps in the $x$ and $y$ coordinate of the ECP shows transition between one point and two point contacts. From Figure~\ref{figure:ex2_3}, one can see that when there is two point contact, the ECP (shown in red) lies on the line joining the two points. When there is a switch to one point contact, i.e., the contact point becomes one of the two black points, the ECP becomes this point. Hence, the $x$ and $y$ coordinate jumps. Similarly, when there is a switch from a single point contact to a two point contact the ECP jumps.

 Figure~\ref{figure:ex2_5} shows the trajectory of the $z$ coordinate of the center of mass of the bar, namely, $q_z$.  Figure~\ref{figure:ex2_6} shows the variation of the velocity of the center of mass, $v_z$, in the top row and the $x-$coordinate of the equivalent contact point (ECP) in the bottom row. Note that $q_z $ is equal to $0.025$m between $1.4s\sim 1.5s$ or after $2.8s$. During those time periods, the T-bar has non-convex surface contact with the ground. At other times, the T-bar has single point or two point contact (see the video: \url{https://youtu.be/T7zV5pEPBeY}).

In Figure\ref{figure:ex2_6}, we juxtapose the two figures to show that the timings of jump in velocities of the center of mass of the T-bar corresponds to the  timings where there is a jump in the $x-$coordinate of the ECP, i.e., the times at which there is an impact due to contact mode change.  Note that the collision is inelastic, so the $z$-component of the velocity at the (actual) impact point goes to zero at impact. The velocity, $v_z$ of the center of mass jumps, but it may not go to $0$. Also, note that we do not track the contact points explicitly, we are actually computing the ECP which always lies on the ground, since we set up the simulation with the object on the ground and applied forces/moments such that the object is always in contact with the ground. In other words, the $z-$component of the velocity of the ECP is always $0$. The velocity $v_z$ is zero during the part of the motion when the T-bar is sliding on the ground. To visualize the different contact modes during the motion, we plot snapshots of some of the contact modes. The timings of the contact modes chosen are shown with red dots in the bottom panel of Figure~\ref{figure:ex2_6}). The corresponding snapshots are shown in  Figure~\ref{Example2_snapshots}.
%veocity of the center   Since our method assumes inelastic contact, there is no bounce in $q_z$ (The increase of $q_z$ after 1.5s is due to the execution of applied torques on T-bar, and it start wobbling on the ground). Besides, the pre-impact velocity $v_z$ also drops to zero. Thus, the contact between T-bar and ground is inelastic. Furthermore, we can observe that the plots for both $q_z$ and $v_z$ are non-smooth. When T-bar wobbles on the ground, it transits from different contact modes frequently. Thus, the trajectory for $q_z$ and $v_z$ become non-smooth.} 

The simulation starts with the T-bar lying flat on the ground with surface contact. The applied torque in phase T1 pivots the T-bar about one vertex as shown in Figure~\ref{figure:ex2_t30}. When the applied torque stops acting, the T-bar falls under the effect of gravity and the pivot point switches to another vertex and there is a period of motion with this new vertex in contact (one such snapshot is shown in Figure~\ref{figure:ex2_t45}). Figure~\ref{figure:ex2_t57} shows that the motion again transitions to a contact mode that is same as shown in Figure~\ref{figure:ex2_t30} and for this transition to happen, there is an intermediate two point contact mode as shown in Figure~\ref{figure:ex2_t53}. Therefore, the ECP lies on the line joining the two contact points as can be seen from the Figure~\ref{figure:ex2_t53}. Note that there is also a two point contact mode (which is not shown in the figure) in going from the pose shown in Figure~\ref{figure:ex2_t30} to the pose in Figure~\ref{figure:ex2_t45}. Thus, the T-bar is rocking back and forth on these two vertices as it falls flat on its face before the phase T2 begins. There are more contact mode transitions that happen in T1 as can be seen from Figure~\ref{figure:ex2_6}, but we have only shown the first few in Figure~\ref{Example2_snapshots}.

The torque applied during phase T2 is such  that $T_x$ is in the opposite direction in T2 compared to phase T1 (see Figure~\ref{figure:ex2_2}). Thus, the motion is similar to the motion in Phase 1, but the contact points are now on the other side of the axis of symmetry of the T-bar. Similar rocking motion occurs with contact mode transitions and some of the contact modes are shown in Figures~\ref{figure:ex2_t170} -~\ref{figure:ex2_t225}. In the phase T3, the T-bar rotates and translates on the plane with surface contact.  During this phase, the ECP changes continuously, as shown in the Figure~\ref{figure:ex2_6}. One snapshot during this motion is shown in Figure~\ref{figure:ex2_t360}. Note that all these transitions were automatically handled by our algorithm.

\subsection{Simulation Scenario with non-convex line contact}
In this example, we simulate a rigid dumbbell moving in contact with a planar support. This example is chosen to illustrate that our method can tackle non-convex line contact where the contact region is topologically disconnected. Furthermore, as the object rolls and rotates, the contact region on the dumbbell changes with the motion. As shown in Figure~\ref{figure:ex3_1}, the planar contact patch between two ends of the dumbbell and the ground is a union of two line segments, which is a non-convex line contact. In Figure~\ref{figure:ex3_2}, we plot the snapshots at each time step for the contact patch when the dumbbell rolls and rotates on the ground. The two line segments (solid black lines) represent the physical contact region, and the convex hull is the entire line (two black lines and the dashed red line in between). Note that, when the dumbbell slides on the ground, the contact line segments on its body stays the same, but when it starts rolling, the line segments change along with the motion. 

The dimensions of the dumbbell are: $L = 0.3$m, $L_b = 0.18$m, $R = 0.1$m, $R_b = 0.05$m, where $L$ is the length of the dumbbell, $L_b$ is the length of the bar, $R$ is the radius of each end, and $R_b$ is the radius of the bar. The mass of the dumbbell is $3$kg, and the other parameters like gravity and friction parameters are the same as in the previous examples. The time step chosen is $h = 0.01$s and the total simulation time is $t = 2.5$s. Figures~\ref{figure:ex3_3} and~\ref{figure:ex3_4} show the external applied impulses on the dumbbell. %Figure~\ref{figure:ex3_4} shows the coordinates of $CM$ during the motion.

Figure~\ref{figure:ex3_2} displays the snapshots for the contact patch during motion. The ECP is marked by a black cross and the projection of CM is marked by a black circle. It can be seen from the figure that the ECP is not within the physical contact region (line segments in black). However, it always lies in the convex hull of the contact regions (on the dashed red line that joins the black contact lines). Furthermore, the ECP has non-zero distance from the projection of the CM due to the fact that the dumbbell is rotating about the global $z$ axis as it is rolling. As shown in the snapshots, initially, the dumbbell has surface contact on the ground (Figure~\ref{figure:ex3_5}). We then exert the applied forces and torques on the dumbbell (shown in Figure~\ref{figure:ex3_3}). The dumbbell falls down with point contact (Figure~\ref{figure:ex3_6})  and moves to a pose with non-convex line contact with union of two line segments (Figure~\ref{figure:ex3_7}). Then it rolls and rotates on the ground (Figures~\ref{figure:ex3_8} and~\ref{figure:ex3_9}). During rolling contact, the contact regions on the object changes continuously. All these transitions were automatically detected by our algorithm.

\section{Conclusion}
\label{se:conclusion}
In this paper we presented a geometrically implicit time-stepping method for solving dynamic simulation problems with planar non-convex contact patches. In our model, we use a convex hull of the non-convex object and combine the collision detection with numerical integration of equations of motion. This allows us to solve for an equivalent contact point (ECP) in the convex hull of the non-convex contact patch as well as the contact wrenches simultaneously. We prove that although we model the contact patch with an ECP, the non-penetration constraints at the end of the time-step are always satisfied. We present numerical simulation for motion prediction for three example scenarios that are representative of applications in robotic manipulation. The results demonstrate  that our method can automatically transition among different contact modes (non-convex contact patch, point, and line).  In the future, we want to use this motion prediction model for developing manipulation planners for moving objects by exploiting contact with the environment.

%%%%%%%%%%%%%%%%%%%%%%%%%%%%%%%%%%%%%%%%%%%%%%%%%%%%%%%%%%%%%%%%%%%%%%
% The bibliography is stored in an external database file
% in the BibTeX format (file_name.bib).  The bibliography is
% created by the following command and it will appear in this
% position in the document. You may, of course, create your
% own bibliography by using thebibliography environment as in
%
% \begin{thebibliography}{12}
% ...
% \bibitem{itemreference} D. E. Knudsen.
% {\em 1966 World Bnus Almanac.}
% {Permafrost Press, Novosibirsk.}
% ...
% \end{thebibliography}

% Here's where you specify the bibliography style file.
% The full file name for the bibliography style file 
% used for an ASME paper is asmems4.bst.
\bibliographystyle{asmems4}

% Here's where you specify the bibliography database file.
% The full file name of the bibliography database for this
% article is asme2e.bib. The name for your database is up
% to you.
%\bibliography{Non_point_contact}

\begin{thebibliography}{10}

\bibitem{ReznikC98}
Reznik, D., and Canny, J., 1998.
\newblock ``A flat rigid plate is a universal planar manipulator''.
\newblock In Proceedings of IEEE International Conference on Robotics and
  Automation, Vol.~2, pp.~1471--1477.

\bibitem{SongTVP04}
Song, P., Trinkle, J., Kumar, V., and Pang, J., 2004.
\newblock ``Design of part feeding and assembly processes with dynamics''.
\newblock In IEEE Intl. Conf. on Robotics and Automation, pp.~39 -- 44.

\bibitem{VoseUL09}
Vose, T.~H., Umbanhowar, P., and Lynch, K.~M., 2009.
\newblock ``Friction-induced lines of attraction and repulsion for parts
  sliding on an oscillated plate''.
\newblock {\em IEEE Transactions on Automation Science and Engineering, {\bf
  6}}(4), Oct, pp.~685--699.

\bibitem{BerardNAT10}
Berard, S., Nguyen, B., Anderson, K., and Trinkle, J., 2010.
\newblock ``Sources of error in a simulation of rigid parts on a vibrating
  rigid plate''.
\newblock {\em ASME Journal of Computational and Nonlinear Dynamics, {\bf
  5}}(4).

\bibitem{XieB+2019}
Xie, J., Bi, C., Cappelleri, D.~J., and Chakraborty, N., 2019.
\newblock ``Towards dynamic simulation guided optimal design of tumbling
  microrobots''.
\newblock In proc. of ASME IDETC \& International Conference on Mechanisms and
  Robotics (IDETC/MR 2019).

\bibitem{DafleR+14}
Dafle, N.~C., Rodriguez, A., Paolini, R., Tang, B., Srinivasa, S.~S., Erdmann,
  M.~A., Mason, M.~T., Lundberg, I., Staab, H., and Fuhlbrigge, T.~A., 2014.
\newblock ``Extrinsic dexterity: In-hand manipulation with external forces''.
\newblock In Proceedings of IEEE International Conference on Robotics and
  Automation (ICRA), pp.~1578--1585.

\bibitem{XieC16}
Xie, J., and Chakraborty, N., 2016.
\newblock ``Rigid body dynamic simulation with line and surface contact''.
\newblock In 2016 IEEE International Conference on Simulation, Modeling, and
  Programming for Autonomous Robots (SIMPAR), pp.~9--15.

\bibitem{XieC18a}
Xie, J., and Chakraborty, N., 2018.
\newblock ``Rigid body dynamic simulation with multiple convex contact
  patches''.
\newblock In proc. of ASME IDETC \& International Conference on Multibody
  Systems, Nonlinear Dynamics, and Control (IDETC/MSNDC 2018).

\bibitem{Haug1986}
Haug, E.~J., Wu, S.~C., and Yang, S.~M., 1986.
\newblock ``Dynamics of mechanical systems with coulomb friction, stiction,
  impact and constraint addition-deletion theory''.
\newblock {\em Mechanism and Machine Theory, {\bf 21}}(5), pp.~401--406.

\bibitem{Cottle2009}
Cottle, R.~W., Pang, J.-S., and Stone, R.~E., 2009.
\newblock {\em The linear complementarity problem}, Vol.~60.
\newblock SIAM.

\bibitem{Trinkle1997}
Trinkle, J.~C., Pang, J.-S., Sudarsky, S., and Lo, G., 1997.
\newblock ``On dynamic multi-rigid-body contact problems with coulomb
  friction''.
\newblock {\em ZAMM-Journal of Applied Mathematics and Mechanics/Zeitschrift
  f{\"u}r Angewandte Mathematik und Mechanik, {\bf 77}}(4), pp.~267--279.

\bibitem{PfeifferG08}
Pfeiffer, F., and Glocker, C., 2008.
\newblock {\em Multibody Dynamics with Unilateral Contacts}.
\newblock Wiley Inc.

\bibitem{XieC18b}
Xie, J., and Chakraborty, N., 2018.
\newblock ``Dynamic model of planar sliding''.
\newblock In Algorithmic Foundations of Robotics (WAFR), The 13th International
  Workshop on the, IFRR.

\bibitem{XieC2019}
Xie, J., and Chakraborty, N., 2019.
\newblock ``Rigid body motion prediction with planar non-convex contact
  patch''.
\newblock In Proceedings of IEEE International Conference on Robotics and
  Automation (ICRA).

\bibitem{MarsdenJM01}
Marsden, J.~E., and West, M., 2001.
\newblock ``Discrete mechanics and variational integrators''.
\newblock {\em Acta Numerica, {\bf 10}}, pp.~357--514.

\bibitem{JohnsonRM09}
Johnson, E.~R., and Murphey, T.~D., 2009.
\newblock ``Scalable variational integrators for constrained mechanical systems
  in generalized coordinates''.
\newblock {\em IEEE Transactions on Robotics, {\bf 25}}(6), p.~1249.

\bibitem{KobilarovCD09}
Kobilarov, M., Crane, K., and Desbrun, M., 2009.
\newblock ``Lie group integrators for animation and control of vehicles''.
\newblock {\em ACM Trans. Graph., {\bf 28}}(2), May, pp.~16:1--16:14.

\bibitem{Facchinei2007}
Facchinei, F., and Pang, J.-S., 2007.
\newblock {\em Finite-dimensional variational inequalities and complementarity
  problems}.
\newblock Springer Science \& Business Media.

\bibitem{Lotstedt82}
Lotstedt, P., 1982.
\newblock ``Mechanical systems of rigid bodies subject to unilateral
  constraints''.
\newblock {\em SIAM Journal on Applied Mathematics, {\bf 42}}(2), pp.~281--296.

\bibitem{AnitescuCP96}
Anitescu, M., Cremer, J.~F., and Potra, F.~A., 1996.
\newblock ``Formulating 3d contact dynamics problems''.
\newblock {\em Mechanics of Structures and Machines, {\bf 24}}(4),
  pp.~405--437.

\bibitem{Pang1996}
Pang, J.-S., and Trinkle, J.~C., 1996.
\newblock ``Complementarity formulations and existence of solutions of dynamic
  multi-rigid-body contact problems with coulomb friction''.
\newblock {\em Mathematical Programming, {\bf 73}}(2), pp.~199--226.

\bibitem{StewartT96}
Stewart, D.~E., and Trinkle, J.~C., 1996.
\newblock ``An implicit time-stepping scheme for rigid body dynamics with
  inelastic collisions and {C}oulomb friction''.
\newblock {\em International Journal of Numerical Methods in Engineering, {\bf
  39}}, pp.~2673--2691.

\bibitem{Liu2005}
Liu, T., and Wang, M.~Y., 2005.
\newblock ``Computation of three-dimensional rigid-body dynamics with multiple
  unilateral contacts using time-stepping and {Gauss}-seidel methods''.
\newblock {\em IEEE Transactions on Automation Science and Engineering, {\bf
  2}}(1), Jan., pp.~19--31.

\bibitem{DrumwrightS12}
Drumwright, E., and Shell, D.~A., 2012.
\newblock ``Extensive analysis of linear complementarity problem (lcp) solver
  performance on randomly generated rigid body contact problems''.
\newblock In 2012 IEEE/RSJ International Conference on Intelligent Robots and
  Systems, pp.~5034--5039.

\bibitem{Todorov14}
Todorov, E., 2014.
\newblock ``Convex and analytically-invertible dynamics with contacts and
  constraints: Theory and implementation in mujoco''.
\newblock In 2014 IEEE International Conference on Robotics and Automation
  (ICRA), pp.~6054--6061.

\bibitem{Studer2009}
Studer, C., 2009.
\newblock {\em Numerics of unilateral contacts and friction: modeling and
  numerical time integration in non-smooth dynamics}, Vol.~47.
\newblock Springer Science \& Business Media.

\bibitem{CapobiancoE2018}
Capobianco, G., and Eugster, S., 2018.
\newblock ``Time finite element based moreau-type integrators''.
\newblock {\em International Journal for Numerical Methods in Engineering, {\bf
  114}}(3), pp.~215--231.

\bibitem{BrulsAC2018}
Br{\"u}ls, O., Acary, V., and Cardona, A., 2018.
\newblock ``On the constraints formulation in the nonsmooth generalized-$
  \alpha $ method''.
\newblock In {\em Advanced Topics in Nonsmooth Dynamics}. Springer,
  pp.~335--374.

\bibitem{AnitescuP02}
Anitescu, M., and Potra, F.~A., 2002.
\newblock ``A time-stepping method for stiff multibody dynamics with contact
  and friction''.
\newblock {\em International Journal for Numerical Methods in Engineering, {\bf
  55}}(7), pp.~753--784.

\bibitem{NilanjanChakraborty2007}
Chakraborty, N., Berard, S., Akella, S., and Trinkle, J., 2014.
\newblock ``A geometrically implicit time-stepping method for multibody systems
  with intermittent contact''.
\newblock {\em The International Journal of Robotics Research, {\bf 33}}(3),
  pp.~426--445.

\bibitem{Brogliato2000}
Brogliato, B., 2000.
\newblock {\em Impacts in mechanical systems: analysis and modelling},
  Vol.~551.
\newblock Springer Science \& Business Media.

\bibitem{Jia2013}
Jia, Y.-B., 2013.
\newblock ``Three-dimensional impact: energy-based modeling of tangential
  compliance''.
\newblock {\em The International Journal of Robotics Research, {\bf 32}}(1),
  pp.~56--83.

\bibitem{Tavakoli+2012}
Tavakoli, A., Gharib, M., and Hurmuzlu, Y., 2012.
\newblock ``Collision of two mass baton with massive external surfaces''.
\newblock {\em Journal of applied mechanics, {\bf 79}}(5).

\bibitem{ChatterjeeRuina1998}
Chatterjee, A., and Ruina, A., 1998.
\newblock ``A new algebraic rigid-body collision law based on impulse space
  considerations''.

\bibitem{AnitescuP97}
Anitescu, M., and Potra, F.~A., 1997.
\newblock ``Formulating dynamic multi-rigid-body contact problems with friction
  as solvable linear complementarity problems''.
\newblock {\em Nonlinear Dynamics, {\bf 14}}(3), pp.~231--247.

\bibitem{Tzitzouris01}
Tzitzouris, J.~E., 2001.
\newblock ``Numerical resolution of frictional multi-rigid-body systems via
  fully implicit time-stepping and nonlinear complementarity''.
\newblock PhD thesis, Johns Hopkins University.

\bibitem{CouBullet}
Coumans, E.
\newblock Bullet physics engine for rigid body dynamics.
\newblock \url{http://bulletphysics.org/}.

\bibitem{SmithODE}
Smith, R.
\newblock Open dynamics engine ode. multibody dynamics simulation software.
\newblock \url{http://www.ode.org/}.

\bibitem{TodorovET2012}
{Todorov}, E., {Erez}, T., and {Tassa}, Y., 2012.
\newblock ``Mujoco: A physics engine for model-based control''.
\newblock In 2012 IEEE/RSJ International Conference on Intelligent Robots and
  Systems, pp.~5026--5033.

\bibitem{TasoraS+2015}
Tasora, A., Serban, R., Mazhar, H., Pazouki, A., Melanz, D., Fleischmann, J.,
  Taylor, M., Sugiyama, H., and Negrut, D., 2015.
\newblock ``Chrono: An open source multi-physics dynamics engine''.
\newblock In International Conference on High Performance Computing in Science
  and Engineering, Springer, pp.~19--49.

\bibitem{LeeG+2018}
Lee, J., Grey, M.~X., Ha, S., Kunz, T., Jain, S., Ye, Y., Srinivasa, S.~S.,
  Stilman, M., and Liu, C.~K., 2018.
\newblock ``Dart: Dynamic animation and robotics toolkit''.
\newblock {\em The Journal of Open Source Software, {\bf 3}}(22), p.~500.

\bibitem{BerardT+2007}
Berard, S., Trinkle, J., Nguyen, B., Roghani, B., Fink, J., and Kumar, V.,
  2007.
\newblock ``davinci code: A multi-model simulation and analysis tool for
  multi-body systems''.
\newblock In Proceedings 2007 IEEE International Conference on Robotics and
  Automation, IEEE, pp.~2588--2593.

\bibitem{rao2005dynamics}
Rao, A., 2005.
\newblock {\em Dynamics of Particles and Rigid Bodies: A Systematic Approach}.
\newblock Cambridge University Press.

\bibitem{Moreau1988}
Moreau, J.~J., 1988.
\newblock ``Unilateral contact and dry friction in finite freedom dynamics''.
\newblock In {\em Nonsmooth Mechanics and Applications}. Springer, pp.~1--82.

\bibitem{Goyal1991}
Goyal, S., Ruina, A., and Papadopoulos, J., 1991.
\newblock ``Planar sliding with dry friction part 1. limit surface and moment
  function''.
\newblock {\em Wear, {\bf 143}}(2), pp.~307--330.

\bibitem{Howe1996}
Howe, R.~D., and Cutkosky, M.~R., 1996.
\newblock ``Practical force-motion models for sliding manipulation''.
\newblock {\em The International Journal of Robotics Research, {\bf 15}}(6),
  pp.~557--572.

\bibitem{trinkle2001dynamic}
Trinkle, J.~C., Tzitzouris, J., and Pang, J.-S., 2001.
\newblock ``Dynamic multi-rigid-body systems with concurrent distributed
  contacts''.
\newblock {\em Philosophical Transactions of the Royal Society of London A:
  Mathematical, Physical and Engineering Sciences, {\bf 359}}(1789),
  pp.~2575--2593.

\bibitem{videos}
Xie, J., and Chakraborty, N.
\newblock Videos of the simulation scenarios.
\newblock \url{https://youtu.be/T7zV5pEPBeY}.

\end{thebibliography}

%%%%%%%%%%%%%%%%%%%%%%%%%%%%%%%%%%%%%%%%%%%%%%%%%%%%%%%%%%%%%%%%%%%%%%
%\appendix       %%% starting appendix
%\section*{Appendix A: Head of First Appendix}
%Avoid Appendices if possible.
\appendix
\section*{Appendix A: Mathematical Background}
\begin{definition}
Let one object be described by the set $F$. Then among all convex sets containing $F$, there exists the smallest one, namely, the intersection of all convex sets containing $F$. This set is called the convex hull of $F$ ($Conv(F)$). 
\end{definition}
\begin{definition}
Given a convex hull of object $Conv(F)$, the extreme points of the convex hull is a point $\bm{x}\in Conv(F)$ with the property that if $\bm{x} = \lambda \bm{y} +(1-\lambda)\bm{z}$ with $\bm{y},\bm{z}\in Conv(F)$ and $\lambda \in [0,1]$, then $\bm{z} = \bm{x}$ and/or $\bm{y} = \bm{x}$.
\end{definition}
%\begin{definition}
% Given a point $\bm{x}$ that lies on the boundary of convex hull $Conv(F)$, and associated active constraints on $Conv(F)$, say $\{f_i(\bm{x}) = 0, \ i \in \mathbb{I}\}$, where $\mathbb{I}$ represents the index set of active constraints. We can define a normal cone, $\mathcal{C}(Conv(F),\bm{x})$, that consists of all vectors in the conic hull of the normals to the surfaces:
%$$\mathcal{C}(Conv(F),\bm{x}) = \{ \bm{y} \vert  \bm{y} = \sum_{i \in \mathbb{I}} \beta_i \nabla f_i(\bm{x}), \beta_i\ge 0\}$$
%where $\beta_i$ are non-negative constants.
%\end{definition}

%\begin{definition}
%\label{def:hyper}
% Given a point $\bm{x}_0$ that lies on the boundary of $Conv(F)$, and it's normal cone $\mathcal{C}(Conv(F),\bm{x}_0) $ on $Conv(F)$,  we can define a supporting hyperplane:
% $$\mathcal{H}(\bm{x}) = \bm{a}^T\bm{x}+b$$ 
% Such that:
% \begin{equation*}
% \begin{aligned}
%\mathcal{H}(\bm{x}_o)= 0, \forall \bm{x}_o\in \partial Conv(F)\\
%\mathcal{H}(\bm{x})\le \mathcal{H}(\bm{x}_o), \forall \bm{x}\in Conv(F)
%\end{aligned}
%\end{equation*}
% where the normal $\bm{a} \in \mathcal{C}(Conv(F),\bm{x}_0)$ and constant $b = -\bm{a}^T\bm{x}_0$.
%\end{definition}

%%%%%%%%%%%%%%%%%%%%%%%%%%%%%%%%%%%%%%%%%%%%%%%%%%%%%%%%%%%%%%%%%%%%%%

\end{document}